\documentclass[conference]{IEEEtran}
\IEEEoverridecommandlockouts
\usepackage{cite}
\usepackage{amsmath,amssymb,amsfonts}
\usepackage{algorithmic}
\usepackage{graphicx}
\usepackage{textcomp}
\usepackage{xcolor}
\def\BibTeX{{\rm B\kern-.05em{\sc i\kern-.025em b}\kern-.08em
    T\kern-.1667em\lower.7ex\hbox{E}\kern-.125emX}}
\usepackage{amsthm,amsmath,amsfonts,amssymb}
\usepackage{mathtools}
\usepackage{algorithmic}
\usepackage{algorithm}

\usepackage{bm}
\usepackage{xcolor}
\usepackage{graphicx}
\usepackage[caption=false, font=footnotesize]{subfig}
\usepackage{tikz}
\usepackage{multirow}
\usepackage{multicol}
\usepackage{hyperref}
\hypersetup{
    colorlinks=true,
    linkcolor=blue,
    citecolor=blue,
    filecolor=magenta,      
    urlcolor=cyan,
}
\usepackage{color, colortbl}
\newcolumntype{P}[1]{>{\centering\arraybackslash}p{#1}}

\newcommand{\oper}[1]{\mathcal{#1}}

\newcommand{\norm}[1]{\left\|#1\right\|}

\newcommand{\E}{\mathbb{E}}

\newcommand{\R}{\mathbb{R}}

\newcommand\numberthis{\addtocounter{equation}{1}\tag{\theequation}}

\newtheorem{ass}{Assumption}
\newtheorem{thm}{Theorem}

\newtheorem{lem}{Lemma}
\newtheorem{prop}{Proposition}

\DeclareMathOperator{\supp}{supp}

\def\bw{\mathbf{w}}
\def\O{\mathcal{O}}

\def\X{\mathcal{X}}
\def\Y{\mathcal{Y}}

\def\bx{\mathbf{x}}
\def\bz{\mathbf{z}}
\def\auc{\text{AUC}}
\def\mbI{\mathbb{I}}

\def\S{\mathcal{S}}
\def\shtauc{\texttt{SHT-AUC}\;}

\begin{document}

\title{Stochastic Hard Thresholding Algorithms for AUC Maximization}

\author{\IEEEauthorblockN{
	Zhenhuan Yang\IEEEauthorrefmark{1},
		Baojian Zhou\IEEEauthorrefmark{2},
		Yunwen Lei\IEEEauthorrefmark{3},
		Yiming Ying\IEEEauthorrefmark{1}
	}
	\IEEEauthorblockA{\IEEEauthorrefmark{1}University at Albany, Albany, USA}
	\IEEEauthorblockA{\IEEEauthorrefmark{2}Stony Brook University, Stony Brook, USA}
	\IEEEauthorblockA{\IEEEauthorrefmark{3}School of Computer Science, University of Birmingham, Birmingham, UK}
	Email: zyang6@albany.edu, baojian.zhou@cs.stonybrook.edu, yunwen.lei@htomail.com,
	yying@albany.edu}

\maketitle

\begin{abstract}
In this paper, we aim to develop stochastic hard thresholding algorithms for the important problem of AUC maximization in imbalanced classification.   The main challenge is the pairwise loss involved in AUC maximization. We overcome this obstacle   by reformulating the U-statistics objective function as an empirical risk minimization (ERM), from which a stochastic hard thresholding algorithm (\texttt{SHT-AUC}) is developed. To our best knowledge, this is the first attempt to provide stochastic hard thresholding algorithms for AUC maximization with a per-iteration cost $\O(b d)$ where $d$ and $b$ are the dimension of the data and the minibatch size, respectively.   We show that the proposed algorithm enjoys the linear  convergence rate up to a tolerance error.  In particular, we show, if the data is generated from the Gaussian distribution, then its convergence becomes slower as the data gets more imbalanced.   We conduct extensive experiments to show the efficiency and effectiveness of the proposed algorithms.  
\end{abstract}


\section{Introduction}
Recently, there are a  considerable amount of work on developing efficient  algorithms for optimizing the Area under the ROC curve  (AUC) score. It is a widely used performance measure for imbalanced data classification  \cite{bradley1997use,fawcett2006introduction,hanley1982meaning,Huang2003} which arises from applications including anomaly detection, information retrieval to cancer diagnosis.  In particular, the work~\cite{Huang2003} showed that the AUC score is, in general, a better  measure than accuracy for evaluating the predictive performance of many data mining algorithms.

In particular, the work by  \cite{Joachims, herschtal2004optimising} employed the cutting plane method and  gradient descent algorithm, respectively.   \cite{Xinhua} developed the Nesterov's accelerated gradient algorithms \cite{nesterov1983method} for optimizing the multivariate performance measures \cite{Joachims}.  The work of \cite{Culver2006} used ideas of  active learning to design heuristic algorithms for AUC maximization. Such optimization algorithms train the model on the whole training data which are not scale well to the high-dimensional data.  Stochastic gradient descent (SGD) algorithms are widely used for high dimensional and large-scale data analysis due to its cheap per-iteration cost. In this aspect,   variants of stochastic (online) gradient descent algorithms have been developed for AUC maximization.  Specifically, \cite{Wang2012,Kar2013,Zhao2011} proposed to a variant of online (projected) gradient descent method. At time $t$, these methods need to compare the current example with previous ones, which have high per-iteration $\O(t d)$ for $d$-dimensional data.  There are various techniques such as using the buffering set to alleviate the bottleneck but the size of the buffer set needs to be large in order to guarantee a good generalization. The appealing work by \cite{gao2013} observed in the case of the least square loss that the updates of these algorithms only rely on the covariance matrix where the per-iteration cost is of $\O(d^2).$ For high dimensional data, it used an appealing low-rank matrix to approximate the covariance matrix in order to reduce the per-iteration costs which may  not  be an ideal solution. Hence, such algorithms have an expensive per-iteration cost, making them not amenable for high dimensional data analysis.  There are some recent work on nonlinear AUC maximization methods such as \cite{gultekin2020mba,khalid2018scalable,liu2019stochastic}. Recently, \cite{ying2016stochastic} reformulated the  AUC maximization problem as a stochastic saddle point (min-max) problem (e.g.  \cite{nemirovski2009robust}), from which a stochastic primal-dual gradient  algorithm was proposed. This algorithm successfully reduced the per-iteration cost to $\O(d).$
\cite{liu2018fast} followed this saddle point formulation for AUC maximizaiton with $\ell_1$ constraints and proposed a fast  multi-stage SGD algorithm. In \cite{Natole2018}, fast SGD-type algorithms were developed for more general strongly convex regularization.

For high-dimensional data analysis, an underlying hypothesis is the sparsity of the data  representation \cite{guyon2002gene,Golub1999un,hromadka2008sparse,candes2008introduction,wright2008robust}. To obtain a sparse solution,  many algorithms have been developed among which the prominent one is based on variants of $\ell_1$-norm constraints (regularization) which includes group lasso \cite{yuan2006model,turlach2005simultaneous}, tree structured group lasso \cite{jacob2009group,jenatton2011structured} etc.  Such approaches are convex and can be solved efficiently by convex optimization.  Concurrently, sparse learning for AUC maximization has been developed  in  \cite{Natole2018,lei2019stochastic,liu2018fast} using $\ell_1$ regularization, where stochastic primal-dual  gradient-type algorithms (SGD) have been developed.  However, as many researchers observed  \cite{duchi2009efficient,langford2009sparse,xiao2010dual},  $\ell_1$-based stochastic algorithms are appealing convex approach which may be  hard to preserve a truly sparse solution. In contrast, the greedy pursuit based on the sparse $\ell_0$ constraints can recover the sparse structure well, among which the most prominent one is  gradient hard thresholding \cite{nguyen2017linear,zhou2018efficient,murata2018sample,Shen2018,liu2017dual,zhou2018efficient}.  In addition, compared to the convex $\ell_1$-norm based methods, hard thresholding algorithms are always orders of magnitude computationally more efficient for large-scale problems \cite{tropp2010computational}.

However, the existing hard thresholding algorithms are developed for the classical regression and classification where the loss is pointwise, i.e. it depends on one data point.  These algorithms can not directly apply to the setting of AUC maximization as its objective function is in the form of U-statistics \cite{clemenccon2008ranking} where the pairwise loss function depends on a pair of data points. 

In this paper, we aim to develop stochastic hard threholding algorithm for the problem of AUC maximization in imbalanced classification.  The main challenge is the pairwise loss involved in AUC maximziaiton. We overcome this obstacle by leveraging the ideas from \cite{Natole2018,ying2016stochastic} by reformulating the U-statistics objective function as a standard empirical risk minimization (ERM). In particular,  the reformulated AUC objective does not necessarily possess the strong convexity property as a whole. Instead,  it is assumed that the objective function obeys the restricted strong convexity and restricted smoothness (RCS/RSS)~\cite{Negahban2009,agarwal2012fast}. The main contribution of the paper is summarized as follows

 $\bullet$ We reformulate the empirical AUC objective in the form of U-statistics as an ERM objective, from which a stochastic hard thresholding algorithm (referred to as \shtauc) is developed. To our best knowledge, this is the first attempt to provide stochastic hard thresholding algorithms for AUC maximization with a per-iteration cost $\O( b\,d)$ where $d$ and $b$ are the dimension of the data and the size of minibatch, repectively.

 $\bullet$  We show that the proposed algorithm enjoys the linear  convergence up to a tolerance error under RCS/RSS properties. We then characterize the RCS/RSS properties in AUC context. In particular, we show, if the data is generated from the Gaussian distribution, that its convergence becomes slower as the data gets more imbalanced, i.e.  the imbalance ratio is getting smaller.

 $\bullet$ We conduct extensive experiements to validate the proposed algorithm (\shtauc) on both simulated and real-world datasets. Our experiments show that the proposed algorithms outperform the existing algorithms in terms of AUC score and the ability of selecting meaningful features.

\noindent{\bf Outline of the paper}.  The rest of this paper are organized as follows.   In Section 2, we reformulate the objective function of AUC maximization, and present the Stochastic Hard Thresholding Algorithm for AUC maximization (i.e. \shtauc). In Section 3, we present its convergence rate and discuss the implication of the theoretical results. In Section 4, we perform experiments on both simulation and real-world datasets to validate the proposed algorithm.   

The detailed proofs for the theoretical results, and the source code of all methods and datasets are available at \url{https://github.com/baojianzhou/sparse-auc}.

\section{Problem Formulation and Proposed Algorithm}

In this section, we introduce necessary notations, formulate the problems of AUC maximization, and present the stochastic hard thresholding algorithm for AUC maximization (\shtauc).

\subsection{Preliminaries}
Given an integer $n \geq 1$, we define $[n] = \{1,...,n\}$.  The standard Euclidean norm of  vector $\bm v = (v_1,...,v_d)^\top \in \mathbb{R}^d$ is denoted by $\|\bm v\|_2 = \sqrt{\sum_{i=1}^{d}v_i^2}$. For any   $\bm v, \bm w \in \mathbb{R}^d$, the inner product is given by $\left<\bm v, \bm w\right> = \sum_{i=1}^d v_i w_i$. The support set of $\bm v$, i.e. indices of non-zeros, is denoted by $\supp(\bm v)$ whose cardinality is written as $\|\bm v\|_0$.  For any integer $d > 0$, suppose that $\Omega$ is a subset of $[d]$. Then for any vector $\bm v \in \mathbb{R}^d$, we define $\mathcal{P}_{\Omega}(\cdot)$ as the orthogonal projection to the support set $\Omega$ which is defined by 
$\left(\mathcal{P}_{\Omega}(\bm v)\right)_i = v_i $ if $ i \in \Omega$ and $0$ otherwise. 
In particular, let $\Gamma$ be the support set indexing the $k$ largest absolute components of $\bm v$. In this way, the hard thresholding operator is given by \begin{align}\label{eq:ht}\mathcal{H}_k(\bm v) = \mathcal{P}_\Gamma (\bm v).\end{align}

Let $\X$ be a domain in $\R^d$ and $\Y= \{\pm 1\}. $ Assume that the training data   $ \S = \{\bz_i= (\bx_i,y_i) \in \X\times \Y:   i \in [n]\}$ is drawn i.i.d from an unknown distribution on  $\mathcal{Z} = \mathcal{X}\times \mathcal{Y}.$ For each $1 \leq i \leq n$, if $y_i = 1$ we say $\bm z_i$ is a positive example otherwise it is a negative example. Let $n_+$ denote the number of positive examples and $n_-$ denote the number of negative examples, and define $r = \frac{n_+}{n}$ as the {\em imbalanced ratio}. Without loss of generality, we assume $n_- \ge n_+$, i.e.  $r \leq 1/2$.

  {\bf Definition of AUC}. 
AUC score \cite{hanley1982meaning,cortes2004auc} measures the probability for a randomly drawn positive instance to have a higher decision value than a randomly sampled negative instance.  Specifically, for any $\bw$, the AUC score on the data $\S$ is defined by 
\begin{align*}\label{eq:auc-def}\auc(\bw) =  \frac{1}{n_+ n_- }\sum_{i,j=1}^n \mbI_{[\bw^\top  (x_i-x_j)> 0]} \mbI_{[y_i=1]}\mbI_{[y_j=-1]},\numberthis
\end{align*} 
where  $\mbI_{[\cdot]}$ is the indicator function which is $1$ for the true event and $0$ otherwise.  The higher the AUC score is, the better performance of the linear function parametrized by $\bw$ will be.  Maximizing the AUC score is equivalent to minimizing $1-\auc(\bw) =  \frac{1}{n_+ n_- }\sum_{i,j=1}^n \mbI_{[\bw^\top  (x_i-x_j)\le  0]} \mbI_{[y_i=1]}\mbI_{[y_j=-1]}.$ In practice, the discontinuous indicator function $\mbI_{[\bw^\top  (x_i-x_j)\le 0]}$ is replaced by a relaxed convex function. As done in \cite{gao2013,ying2016stochastic,Natole2018,liu2018fast}, in this paper we restrict our attention to the least square loss, i.e. replacing  $\mbI_{[\bw^\top  (x_i-x_j)\le 0]}$ by $(1- \bw^\top(x_i-x_j))^2.$ 

Now sparse AUC maximization with $\ell_0$ constraints  is given by 
\begin{align}\label{eq:sparse-auc}
 \min_{\|\bw\|_0\le k_*}  \frac{1}{n_+ n_-}\sum_{i,j=1} (1- \bw^\top(x_i-x_j))^2 \mbI_{[y_i =1]} \mbI_{[y_j=-1]}.
\end{align}
The objective function $F(\bw) = \frac{1}{n_+n_-}\hspace*{-1mm}\sum_{i,j=1} \hspace*{-0.5mm}(1- \bw^\top(x_i-x_j))^2 \mbI_{[y_i =1]} \mbI_{[y_j=-1]}$ is the average of pairwise losses and has the form of U-statistics \cite{clemenccon2008ranking}. 

  {\bf Objective and Challenges.} Our main objective in this paper is to develop efficient stochastic optimization algorithms for the sparse AUC maximization formulation \eqref{eq:sparse-auc} which is scalable to large scale and high-dimensional imbalanced data.

One possible approach is to directly apply stochastic hard thresholding algorithms \cite{shen2016online,nguyen2017linear,zhou2018efficient} to the setting of AUC maximization by regarding pairs of examples as individual ones, i.e. at each time we randomly samples a pair of examples or a minibatch of pairs to update model parameters.   However, this means that one pass of the data, i.e. passing all pairs,  will require $n$ passes of the original dataset $\S$ which makes it not suitable for large-scale and high-dimensional data analysis. To address this challenge, we show in the following subsections that this multiple passes can be avoided by  reformulating the minimization problem of pairwise U-statistics objective function $F(\bw)$ as a novel ERM formulation. From this new reformulation, we can develop efficient stochastic hard thresholding algorithms for AUC maximization.

\subsection{Equivalent Reformulation}
Inspired by the work \cite{ying2006online,Natole2018,lei2019stochastic}, we will formulate the U-statistics objective function in \eqref{eq:sparse-auc} as an ERM objective function, i.e.  singled-summed objective function. For this purpose, let the positive and negative sample mean be respectively denoted by 
$\overline{\bm x}_+ = \frac{\sum_{i=1}^n \bm x_i\mathbb{I}_{[y_i=1]}}{n_+},~\overline{\bm x}_- = \frac{\sum_{i=1}^n \bm x _i\mathbb{I}_{[y_i=-1]}}{n_-}.$
Then, we have the following proposition.  
\begin{prop}\label{prop:formulation}
The empirical AUC objective function $F(\bw)$ given by \eqref{eq:sparse-auc} can be reformulated as 
\begin{equation}
\label{eq:single-sum-AUC}    
F(\bm w) = \frac{1}{n}\sum_{i=1}^{n} \widetilde{f}(\bm w; \bm z_i)
\end{equation}
where $\widetilde{f}(\bm w; \bm z_i) = \frac{1}{r}(\bm w^\top (\bm x_i - \overline{\bm x}_+))^2\mathbb{I}_{[y_i=1]}
  + \frac{1}{1-r}(\bm w^\top (\bm x_i - \overline{\bm x}_-))^2\mathbb{I}_{[y_i=-1]}  
  +1 + 2\bm w^\top (\overline{\bm x}_- - \overline{\bm x}_+)+(\bm w^\top (\overline{\bm x}_- - \overline{\bm x}_+))^2.$
\end{prop}
 
The proof of Proposition~\ref{prop:formulation} is inspired by \cite{ying2016stochastic,Natole2018}.  However, the original proofs there need to introduce three auxiliary variables. Our proof is  much simpler and straightforward without introducing auxiliary variables.

Assume that $n$ can be divided by $m$ and let the black size $b= n/m$. Then,  in order to apply minibatch updates, let $\{B_i: i \in [m]\}$ denote non-overlapping subsets of $\S$, each of which is of size $b$. Therefore, the U-statistics form in the AUC maximization problem \eqref{eq:sparse-auc} , with $m = n/b$, can be formulated as the following ERM with sparse constraints: 
\begin{equation}\label{eq:decomposed-objective}
   	 \bw_* = \arg\min_{\|\bw\|_0 \leq k_*} F(\bm w) = \frac{1}{m}\sum_{i=1}^{m}f_{B_i}(\bm w).
\end{equation}
where $f_{B_i}(\bm w) = \frac{1}{b}\sum_{j\in B_i}\widetilde{f}(\bm w; \bm z_j)$ is a block objective.

\subsection{The \texorpdfstring{$\shtauc$}  Algorithm}

From the formulation \eqref{eq:decomposed-objective},  we are ready to  present  the Stochastic   Hard Thresholding AUC Optimization algorithm which is referred to as \texttt{SHT-AUC}.

\begin{algorithm}
\caption{\shtauc Algorithm}
\label{alg:SG-HT}
   	\begin{algorithmic}
   		\STATE {\bfseries Input:} Relaxed sparsity level $k$, step size $\gamma$, initial classifier $\bm w_0$ such that $\|\bm w_0\|_0 \leq k$
   		\STATE {\bfseries Compute:} $\overline{\bx}_+$ and  $ \overline{\bx}_- $
   		\FOR{$t=0$ to $T-1$}
   		\STATE Randomly selected $i_t \in [m]$
      	\STATE $\bm w_{t+1} = \mathcal{H}_{k}\left(\bm w_t - \gamma \nabla f_{B_{i_t}}(\bm w_t)\right)$
   		\ENDFOR
   		\STATE {\bfseries Output:} $\bm w_{T}$
	\end{algorithmic}
  \end{algorithm}

The pseudo-code is given by  Algorithm~\ref{alg:SG-HT} which is taken from \cite{Nguyen2017,needell2009cosamp,Shen2018}.  It can be regarded as "expansive" projected SGD with projections to the $\ell_0$ constraints. Specifically,  at each iteration,  it randomly selects $i_t$ from $[m]$ with probability $\frac{1}{m}$, and hence the minibatch $B_{i_t}.$ Then, the current model parameter $\bw_t$ is updated using projected gradient descent based on the gradient $\nabla f_{B_{i_t}}$, which is the hard thresholding operator given by~\eqref{eq:ht}. The main computation is $\mathcal{O}(bd)$ with the gradient and $\mathcal{O}(d)$ with the hard thresholding. Hence the per-iteration cost is $\mathcal{O}(bd)$.  To further explain why the Hard Thresholding Operator $\operatorname{H}_k(\bm w):\mathbb{R}^d\mapsto \mathbb{R}^d$ in Algorithm~\ref{alg:SG-HT} only needs time complexity $\O(d)$, firstly we consider a common choice to fulfill $\operatorname{H}_k$ – use sorting algorithm of $\bm w$ with the time complexity $\O(d\log(d))$ in expectation, such as quick sort, then pick the largest $k$ entries. In fact, given $\bm w \in \mathbb{R}^d$ and the sparsity $k \ll d $, the Hard Thresholding Operator can be computed as following
\begin{equation}
\mathcal{H}_k(\bm w) = \bm w \cdot \mbI_{|\bm w| \geq |w_{\tau_k}\bm 1_d|}, \label{inequ:hard-thresholding}
\end{equation}
where $\tau_k$ is the index of $k$-th largest magnitude among $|w_1|, |w_2|,$ $\ldots, |w_d|$ and all operations in~(\ref{inequ:hard-thresholding}) are element-wise. Clearly, if we know $w_{\tau_k}$ in advance, we can get the solution of the operator in $\Theta(d)$. To find $w_{\tau_k}$, we choose the Floyd-Rivest algorithm~\cite{blum1973time} with time complexity $\O(d)$ in expectation. Algorithm~\ref{alg:folyd-rivest} is the pseudo-code of the Floyd-Rivest method. One can immediately find $|w_{\tau_k}| = |w_k|$ after call $\textbf{Select}(\bm w,0,d-1,k)$.

Another note is here the sparsity level $k$ in \shtauc is not necessary to be $k_*$. In particular, the flexible choice of  of $k\ge k_\ast$ follows the appealing work  \cite{Jain2014,Shen2018}, which  will allow a relaxed projection to the $\ell_0$ constraints, and therefore lead to tighter bounds as we can see below.  

\begin{algorithm*}
\caption{\textbf{Select}$({\bm w}, l, r, k)$: Floyd-Rivest Algorithm~(\cite{blum1973time})}
\begin{multicols}{2}
\begin{algorithmic}[1]
\WHILE{$r > l$}
\IF{$r - l > 600$}
\STATE $n = r - l + 1$; $i = k - l + 1$;
\STATE $z = \ln(N)$; $s = 0.5 * \exp(2*z/3)$;
\STATE $sd = 0.5 * \sqrt{z*s*(n-s)/n}*\operatorname{sign}(i-n/2)$;
\STATE $ll = \max(l,k-i*s/n + sd)$;
\STATE $rr = \min(r,k+(n-1)*s/n + sd)$;
\STATE $\textbf{Select}(\bm w, ll, rr, k)$;
\ENDIF
\STATE $t = w_k$; $i = l$; $j = r$; $\operatorname{swap}(w_l,w_k)$
\WHILE{$i < j$}
\STATE $\operatorname{swap}(w_i,w_j)$; $i=i+1$; $j=j-1$;
\WHILE{$w_i < t$}
\STATE $i = i+1$;
\ENDWHILE
\WHILE{$w_j < t$}
\STATE $j = j-1$;
\ENDWHILE
\ENDWHILE
\IF{$w_l = t$}
\STATE $\operatorname{swap}(w_l,w_j)$;
\ELSE
\STATE $j = j+1$; $\operatorname{swap}(w_j,w_r)$;
\ENDIF
\IF{$j \leq k$}
\STATE $l = j+1$;
\ENDIF
\IF{$k \leq j$}
\STATE $r = j-1$;
\ENDIF
\ENDWHILE
\end{algorithmic}\label{alg:folyd-rivest}
\end{multicols}
\end{algorithm*}

\section{Convergence Analysis}
In this section, we turn to the convergence analysis of \shtauc algorithm. The convergence typically need the following standard assumptions. 

\begin{ass}\label{ass:rsc}
	The function $F(w)$ satisfies the $\rho_k^-$-restricted strong convexity (RSC) condition if there exists a positive constant $\rho_k^-$ such that 
	\begin{equation}\label{eq:rsc}
		F(\bw') - F(\bw) - \left<\nabla F(\bw),\bm w'-\bw\right> \geq \frac{\rho^-_k}{2}\|\bm w'-\bw\|_2^2
	\end{equation} for any  $\bw$ and $\bw'$ such that $|\supp(\bw) \cup \supp(\bw')| \leq k$.
\end{ass}

\begin{ass}\label{ass:rss}
For all $1 \leq i \leq m$, the function $f_{B_i}(w)$ satisfies the $\rho_k^+$-restricted strong smoothness (RSS) condition if there exists a positive constant $\rho_k^+$ such that 
\begin{equation}\label{eq:rss}
	\|\nabla f_{B_i}(\bm w) - \nabla f_{B_i}(\bm w')\|_2 \leq \rho^+_k \|\bm w-\bm w'\|_2
\end{equation} for all vectors $\bm w$ and $\bm w'$ such that $|\supp(\bm w) \cup \supp(\bm w')|\leq k$.

\end{ass}  

RSC/RSS properties are firstly introduced in~\cite{Negahban2009}. Since it captures sparsity of many functions, it has been widely used for designing sparsity constrained algorithms~\cite{Jain2014,nguyen2017linear,Shen2018,zhou2018efficient,elenberg2018restricted}. Here, we define the {\em $k$-restricted condition number} to be $\rho_k = \rho_k^+/\rho_k^-$.

In the sequel, we will first state the convergence results related to the RSC and RSS conditions. Then we will prove that the objective function of AUC maximization defined by $F(\bw)$ satisfies the RSC and RSS conditions and discuss their  implications on the convergence of \shtauc.

\subsection{General Convergence Results}   
To state the convergence of \shtauc recall that $k_*$ is the desired sparsity level and $k$ is the relaxed sparsity level. We are now ready to state the general convergence result of \shtauc. 

\begin{thm}
\label{thm:StoIHT}
Let $w_*$ be a $k_*$-sparse vector of interest, and $w_0$ be the initial solution. Consider the problem~\eqref{eq:decomposed-objective} with sparsity level $k$ such that $d \gg k > (\rho_{2k+k_*}^2 - \rho_{2k+k_*})k_*$. Select $\gamma = \frac{1}{\rho_{2k+k_*}^+}$ and let $\nu = 1 + k_*/k+\sqrt{k_*/k}$, we have, 
\begin{equation}
\mathbb{E} \left\|\bm w_{t+1} - \bm w_*\right\|_2 \leq {\kappa}^{t+1} \left\|\bm w_0 - \bm w_*\right\|_2  + \frac{{\sigma}_{\bm w_*} }{1-{\kappa}} 
\end{equation}
where the expectation is taken over all choices of random variables $i_0,...,i_t$. Here 
\begin{equation}
\label{eq:new kappa of SG-HT}
\kappa = \sqrt{\nu\left(1-1/\rho_{2k+k_*}\right)} < 1 
\end{equation} is the convergence parameter and
\begin{equation}
\label{eq:new sigma of SG-HT}
{\sigma}_{\bm w_*} = \frac{\gamma}{m} \sqrt{\nu}  \sum_{i = 1}^{m} \max_{|\Omega|\leq 2k+k_*} \left\| \mathcal{P}_{\Omega}\left(\nabla f_{B_{i}}(\bm w_*)\right)\right\|_2
\end{equation}
is the tolerance error parameter.
\end{thm}

Theorem~\ref{thm:StoIHT} shows that Algorithm~\ref{alg:SG-HT} still possibly enjoys linear
convergence up to the {\em tolerance error} $\sigma_{\bw_*}/(1 - \kappa).$ As indicated in Theorem~\ref{thm:StoIHT}, both the convergence rate and the error are depending on the restricted condition number $\rho_{2k+k_*}$. In the next subsection, we will characterize $\rho_{2k+k_*}$ in terms of imbalance ratio $r$. 

\subsection{Estimation of RCS and RSS Conditions}
\label{subsection:Estimation of RCS and RSS Conditions}
In this subsection, we will estimate the  RSC and RSS conditions, and the  condition number for AUC maximization. Combining this with Theorem  \ref{thm:StoIHT}, we will discuss the implications of these estimations, particularly on  the effect of imbalance ratio $r = \frac{n_+}{n}$ on the convergence of \shtauc.  

For our analysis, we assume each $\bm x_i \in \mathbb{R}^d$ are i.i.d Gaussian random vectors from $\mathcal{N}(0,\Sigma)$ with covariance matrix  $\Sigma$ and its diagonal elements  satisfying $ \Sigma_{jj} \leq 1$. We also define a shorthand notation $\lambda = \lambda_{\min}(\Sigma^{1/2}).$ Now we have the following theorem.

\begin{thm}\label{thm:rip}
Consider objective function of AUC maximization given by \eqref{eq:decomposed-objective} and Algorithm \ref{alg:SG-HT} with sparsity level $k<d$. The RCS/RSS condition is satisfied and we have following results:
With probability at least $1 - \exp(-n^+/72) - 2(2k+k_*)/d$, there holds 
\begin{align*}\label{eq:rcs-gaussian}
    \rho^-_{2k+k_*} = & \left(\frac{1}{2}\lambda -  6\sqrt{2}\sqrt{\frac{(2k+k_*)\log{d}}{rn}}\right)^2\\
    & - \frac{32}{3}\frac{(2k+k_*)\log (d)}{rn}.\numberthis
\end{align*} 
With probability at least  $1-(2k+k_*)/2d$, we have  \begin{equation}\label{eq:rss-gaussian}
	    \rho^+_{2k+k_*} = \frac{16(2k+k_*) \log(d)\left(\frac{1}{2}\log(b) + \log(d)\right)}{r}.
	\end{equation}
\end{thm}

\textbf{Imbalance ratio on the contraction coefficient.}  
Since we focus on the case of large scale problem, we can assume the number of total examples $n$ is large enough (mainly negative examples $n_-$) such that $\rho_k^-$ is positive. We can write the restricted condition number $\rho_k$ as a function of the imbalance ratio $r$,
\begin{equation}
\label{eq:gaussian-rho}
\rho_k(r) = \frac{16}{ar + b\sqrt{r} + c},
\end{equation}
where the coefficients are $a = \frac{\lambda^2}{4k\log(d)\left(\frac{1}{2}\log(b) + \log(d)\right)}$, $b = -\frac{6\sqrt{2}\lambda}{\sqrt{nk\log(d)\left(\frac{1}{2}\log(b) + \log(d)\right)}}$ and $c = \frac{184}{3n\left(\frac{1}{2}\log(b) + \log(d)\right)}$. The bottom of~\eqref{eq:gaussian-rho} is a concave quadratic function of $\sqrt{r}$ with its minimum attaining at the axis of symmetry: $\sqrt{r_*} = \frac{12\sqrt{2}\sqrt{k\log(d)\left(\frac{1}{2}\log(b) + \log(d)\right)}}{\lambda\sqrt{n}}.$
Since we consider the regime when $n$ (or $n_-$) sufficiently large, i.e. the axis of symmetry is close to $0$. Therefore $\rho_k(r)$ can be regarded as a monotonically decreasing function of $\sqrt{r}$. Recall in Theorem~\ref{thm:StoIHT} equation~\eqref{eq:new kappa of SG-HT}, $\kappa = \Omega(\sqrt{1 - 1/\rho_{2k+k_*}})$ is monotonically increasing with respect to $\rho_{2k+k_*}$. Therefore $\kappa$ is also a monotonically decreasing function of $r$.

\textbf{Imbalance ratio on the tolerance error.} Recall that the step-size in Theorem~\ref{thm:StoIHT} is chosen as $\gamma = 1/\rho_{2k+k_*}^+$. Hence, the tolerance parameter in Theorem~\ref{thm:StoIHT} equation~\eqref{eq:new sigma of SG-HT} is of the form ${\sigma}_{\bm w_*} = \Omega(1/\rho_{2k+k_*}^+)$
Now, combine the discussion on the contraction parameter $\kappa$, the total tolerance error, after simplification, is of the form
\begin{equation}
\label{eq:gaussian-tolerance-error}
\frac{\sigma_{\bm w_*}}{1 - \kappa} \geq \frac{c_1 r}{1 - c_2\sqrt{1 - r}}
\end{equation}
for some constant $c_4$ and $c_5$. See the exact values in the detailed proofs. Combining this with the above discussion on the relation between $r$ and $\kappa$,   we can conclude that \textit{the more imbalance the data is, the slower the convergence is, and the larger tolerance error is}, which matches the empirical experience in the subsequent section.

\section{Experiments}

To validate the effectiveness of our proposed \shtauc and test our theory, we apply it to both synthetic and real-world datasets.


\textbf{Baseline methods.}\footnote{We did not consider methods such as OAM~\cite{Zhao2011} and OPAUC~\cite{gao2013} due to their inferior performance on both run time and AUC score reported in~\cite{ying2016stochastic,liu2018fast}.}  We consider six baseline methods which can be divided into two kinds. The first kind is methods that directly optimize the AUC objective. It includes: \textsc{SOLAM}, a Stochastic OnLine algorithm for AUC Maximization proposed in~\cite{ying2016stochastic}; \textsc{SPAM}-based, a stochastic proximal algorithm for AUC maximization designed in~\cite{Natole2018}. Based on different regularizations, we refer \textsc{SPAM} using $\ell_1$ and $\ell^2$ as \textsc{SPAM}-$\ell_1$, \textsc{SPAM}-$\ell^2$ respectively; \textsc{FSAUC}, a Fast Stochastic algorithm for true \textit{AUC} maximization as proposed in~\cite{liu2018fast}.  The second kind is algorithms that optimize the logistic loss with $\ell_0$-norm constraint. We consider two popular methods of this type including \textsc{StoIHT}, a Stochastic Iterative Hard Thresholding method defined~\cite{Nguyen2017} and \textsc{HSG-HT}, a Hybrid Stochastic Gradient Hard Thresholding~\cite{zhou2018efficient} algorithm. 

\textbf{Evaluation Metrics.} One of the main goals is to testify the effectiveness of optimizing AUC score and the feature selection ability. We use AUC score~\cite{hanley1982meaning} for the classification performance and use F1 score for the feature selection. The F1 score with respect to $\bm w_t$ and $\bm w_*$ is defined as
$$
\operatorname{F1}(\bm w_t, \bm w_*) = \frac{2\operatorname{Pre}(\bm w_t, \bm w_*)\cdot \operatorname{Rec}(\bm w_t, \bm w_*)}{\operatorname{Pre}(\bm w_t, \bm w_*)+\operatorname{Rec}(\bm w_t, \bm w_*)},
$$
where $ 
\operatorname{Pre}(\bm w_t, \bm w_*) = \frac{|\operatorname{supp}(\bm w_*)\cap \operatorname{supp}(\bm w_t)|}{\|\bm w_t\|_0}$ and 
$\operatorname{Rec}(\bm w_t, \bm w_*) = \frac{|\operatorname{supp}(\bm w_*)\cap \operatorname{supp}(\bm w_t)|}{\|\bm w^*\|_0}.    
$ We also use the Jaccard Index as an alternative metric to evaluate the feature selection ability. Jaccard Index (JI) with respect to $\bm w_t$ and $\bm w_*$ is defined as
$$
\operatorname{JI}(\bm w_t, \bm w_*) = \frac{|\operatorname{supp}(\bm w_*)\cap \operatorname{supp}(\bm w_t)|}{|\operatorname{supp}(\bm w_*)\cup \operatorname{supp}(\bm w_t)|}.
$$

\subsection{Synthetic Datasets}
\textbf{Data generation.} We first generate simulation datasets with the data size $n=1000$ and dimension $d=1000$. This simulation is motivated from the task of disease outbreak detection~\cite{arias2011detection}. More specifically, for each of the datasets, each training sample $\bm x \in \mathbb{R}^{d}$ and $y \in \{\pm 1\}$. All entries of each negative sample are from $\mathcal{N}(0,1)$ while each positive training sample is generated according to: 
$\bm x_i\sim \mathcal{N}(\mu,1)$ if  $i \in S$  and $x_i \sim \mathcal{N}(0,1)$ if  $i \notin S.$
Here we fix $\mu=0.3$ and $S$ is  a subset of ``important features'' as the ground truth features. They are randomly selected from $\{0,1,2,\ldots,999\}$ and the size of $S$ is treated as the true sparsity $k_*$. We generate datasets that have different true sparsity $k_* \in \{20:20:80\}$ and different imbalance ratios $r \in \{0.05:0.05:0.5\}$.
   
\textbf{Parameter Tuning.} For \textsc{SOLAM}, it has two parameters, the bound $R\in 10^{[-1:1:5]}$ on $\bm w$ and the initial learning rate $\xi \in [1:9:100]$ as suggested; For the \textsc{SPAM}-$\ell_1$ method, it has $\ell_1$-regularization parameter $\beta_1 \in 10^{[-5:1:2]}$; Similarly, \textsc{SPAM}-$\ell^2$ has the $\ell^2$-regularization parameter $\beta_2 \in 10^{[-5:1:2]}$ or both. The initial learning rate $\xi$ of \textsc{SPAM}-$\ell_1$ and \textsc{SPAM}-$\ell_2$ is the same as $\xi$ in \textsc{SOLAM}; For \textsc{FSAUC}, the initial step size $\eta_1 $ is tuned from $ 2^{[-10:1:10]}$ and the bound parameter $R$ of $\bm w$ is the same as \textsc{SOLAM}. For three non-convex methods, the sparsity parameter $k$ of $\shtauc$, \textsc{StoIHT}, and \textsc{HSG-HT} is tuned from $k \in \{10:10:100\}$. The number of blocks of $\shtauc$ and \textsc{StoIHT} is from $\{1,2,4,8,10\}$.

\begin{table}[!ht]
\centering
\small
\begin{tabular}{@{\hskip1pt}c@{\hskip1pt}|c|c|c|c}
\hline\hline
 & $k_*=20$ & $k_*=40$ & $k_*=60$ & $k_*=80$\\
\hline 
\shtauc & .551$\pm$.107 & \textbf{.675$\pm$.068} & \textbf{.766$\pm$.074} & \textbf{.820$\pm$.061} \\\hline
SPAM-$\displaystyle \ell ^{1}$ & .560$\pm$.087 & .621$\pm$.094 & .697$\pm$.118 & .763$\pm$.128 \\\hline
SPAM-$\displaystyle \ell ^{2}$ & .537$\pm$.095 & .597$\pm$.110 & .653$\pm$.141 & .752$\pm$.135 \\\hline
FSAUC & \textbf{.571$\pm$.107} & .654$\pm$.083 & .754$\pm$.079 & .820$\pm$.071 \\\hline
SOLAM & .523$\pm$.102 & .628$\pm$.077 & .732$\pm$.092 & .740$\pm$.139 \\\hline
StoIHT & .538$\pm$.096 & .604$\pm$.087 & .659$\pm$.091 & .719$\pm$.092 \\\hline
HSG-HT & .484$\pm$.094 & .593$\pm$.089 & .661$\pm$.116 & .759$\pm$.083 \\\hline
\end{tabular}\caption{Averaged AUC on four synthetic datasets.}
\label{tab:auc}
\end{table}

\begin{table*}[!ht]
\centering
\small
\begin{tabular}{m{0.08\textwidth}|m{0.08\textwidth}|m{0.08\textwidth}|m{0.08\textwidth}|m{0.08\textwidth}|m{0.08\textwidth}|m{0.08\textwidth}|m{0.08\textwidth}|m{0.08\textwidth}}
\hline\hline
& \multicolumn{4}{c}{F1 score} & \multicolumn{4}{c|}{Jaccard Index}\\
\hline 
  & $k_*=20$ & $k_*=40$ & $k_*=60$ & $k_*=80$ & $k_*=20$ & $k_*=40$ & $k_*=60$ & $k_*=80$ \\
\hline 
\shtauc & \textbf{.209$\pm$.046} & \textbf{.365$\pm$.078} & \textbf{.382$\pm$.053} & \textbf{.450$\pm$.072} & \textbf{.126$\pm$.047} & \textbf{.200$\pm$.043} & \textbf{.275$\pm$.051} & \textbf{.311$\pm$.032} \\\hline
SPAM-$\displaystyle \ell ^{1}$ & .058$\pm$.053 & .182$\pm$.137 & .147$\pm$.102 & .177$\pm$.126 & .028$\pm$.040 & .087$\pm$.076 & .078$\pm$.060 & .101$\pm$.059 \\\hline
SPAM-$\displaystyle \ell ^{2}$ & .037$\pm$.019 & .060$\pm$.053 & .100$\pm$.085 & .159$\pm$.065 & .017$\pm$.021 & .029$\pm$.020 & .040$\pm$.034 & .065$\pm$.033 \\\hline
FSAUC & .100$\pm$.075 & .202$\pm$.114 & .222$\pm$.096 & .320$\pm$.102 & .037$\pm$.033 & .117$\pm$.071 & .146$\pm$.060 & .210$\pm$.069 \\\hline
SOLAM & .044$\pm$.021 & .088$\pm$.039 & .125$\pm$.064 & .171$\pm$.064 & .024$\pm$.029 & .049$\pm$.032 & .073$\pm$.027 & .100$\pm$.072 \\\hline
StoIHT & .089$\pm$.037 & .163$\pm$.069 & .231$\pm$.054 & .237$\pm$.067 & .051$\pm$.033 & .093$\pm$.028 & .122$\pm$.036 & .146$\pm$.045 \\\hline
HSG-HT & .089$\pm$.042 & .157$\pm$.076 & .228$\pm$.061 & .249$\pm$.066 & .046$\pm$.037 & .096$\pm$.031 & .127$\pm$.049 & .171$\pm$.042\\\hline
\end{tabular}\caption{Averaged F1 scores and Jaccard Index on four synthetic datasets.}
\label{tab:f1-score-jaccard-index}
\end{table*}

{\bf Generalization Performance and Feature Selection.} Table~\ref{tab:auc} reports the averaged AUC of four datasets with imbalanced ratio $r=0.05$\footnote{AUC scores are calculated on testing dataset. We found that \textsc{SPAM}-based, \textsc{SOLAM} and \textsc{FSAUC} do not produce sparse solutions. Instead, we truncate all entries in $\bm w_t$ to 0 if the magnitude of these entries are not larger than $0.001$.}. First of all, \shtauc gives the best AUC score for $k_* = 40, 60, 80$ and gives competitive AUC score when $k_*=20$. In fact the AUC scores of \shtauc and \textsc{FSAUC} are competitive with each other. One of the reasons could be that both of them have a sparse projection at each iteration. Secondly, the AUC scores of \textsc{SPAM}-$\ell^2$ and \textsc{SOLAM} are inferior to \shtauc, \textsc{SPAM}-$\ell_1$, and \textsc{FSAUC}. This is because these two are not sparse-inducing methods hence not suitable for sparse learning problem. Last but not least, $\ell_0$-based methods including \textsc{StoIHT} and \textsc{HSG-HT} have lower AUC scores since these two are not for AUC optimization. It shows that our algorithm generalizes well by solving~\eqref{eq:sparse-auc} when the ground truth $w_*$ is sparse. 

Table~\ref{tab:f1-score-jaccard-index} reports the average F1 score and Jaccard Index of four datasets with imbalanced ratio $r=0.05$. In both metric, our method \shtauc are significantly better than any other algorithms. This impact is two-fold. Firstly, \shtauc is better than other $\ell_1$ based stochastic AUC maximization algorithms. This is consistent with the fact $\ell_1$ based stochastic  algorithms may be hard to preserve a truly sparse solution~\cite{duchi2009efficient,langford2009sparse,xiao2010dual}. And it shows the advantage of using $\ell_0$ based stochastic algorithm as \shtauc. Secondly, \shtauc is better than $\textsc{StoIHT}$ and $\textsc{HSG-HT}$. The advantage of directly optimizing AUC compared with using Empirical Risk Minimization, i.e. logistic loss, when the dataset is imbalanced can also be found in~\cite{cortes2004auc}. Our findings prove this well. 

In summary, the simulation results indicate that \shtauc has better tradeoff between the AUC optimization and feature selection among all methods when the data is imbalanced and the ground truth is sparse.

{\bf Effect of Imbalanced Ratio on Convergence and Performance.} To demonstrate the impact of imbalance ratio $r$ on the convergence of \shtauc, we apply \shtauc on datasets with different imbalance ratios $r=0.05, 0.25, 0.50$. Figure~\ref{fig:conv} reports the number of epochs against the AUC score, with sparsity level $k_* = 20, 40, 60, 80$ and fix $k = k_*$ and batch size $b = 50$. Note the AUC scores are scaled in order to get better visualization. We can observe that when $r = 0.5$ the \shtauc converges after 150 epochs, but when $r=0.05$, \shtauc does not converge even after 300 epochs. This results proves our theoretical analysis in Section~\ref{subsection:Estimation of RCS and RSS Conditions}, i.e. when data is more imbalanced, the convergence is slower. It also matches ones empirical expectation.

\begin{figure}[ht!]
\centering
\subfloat[$k_*=20$\label{fig:conv-20}]{\includegraphics[width=0.5\linewidth]{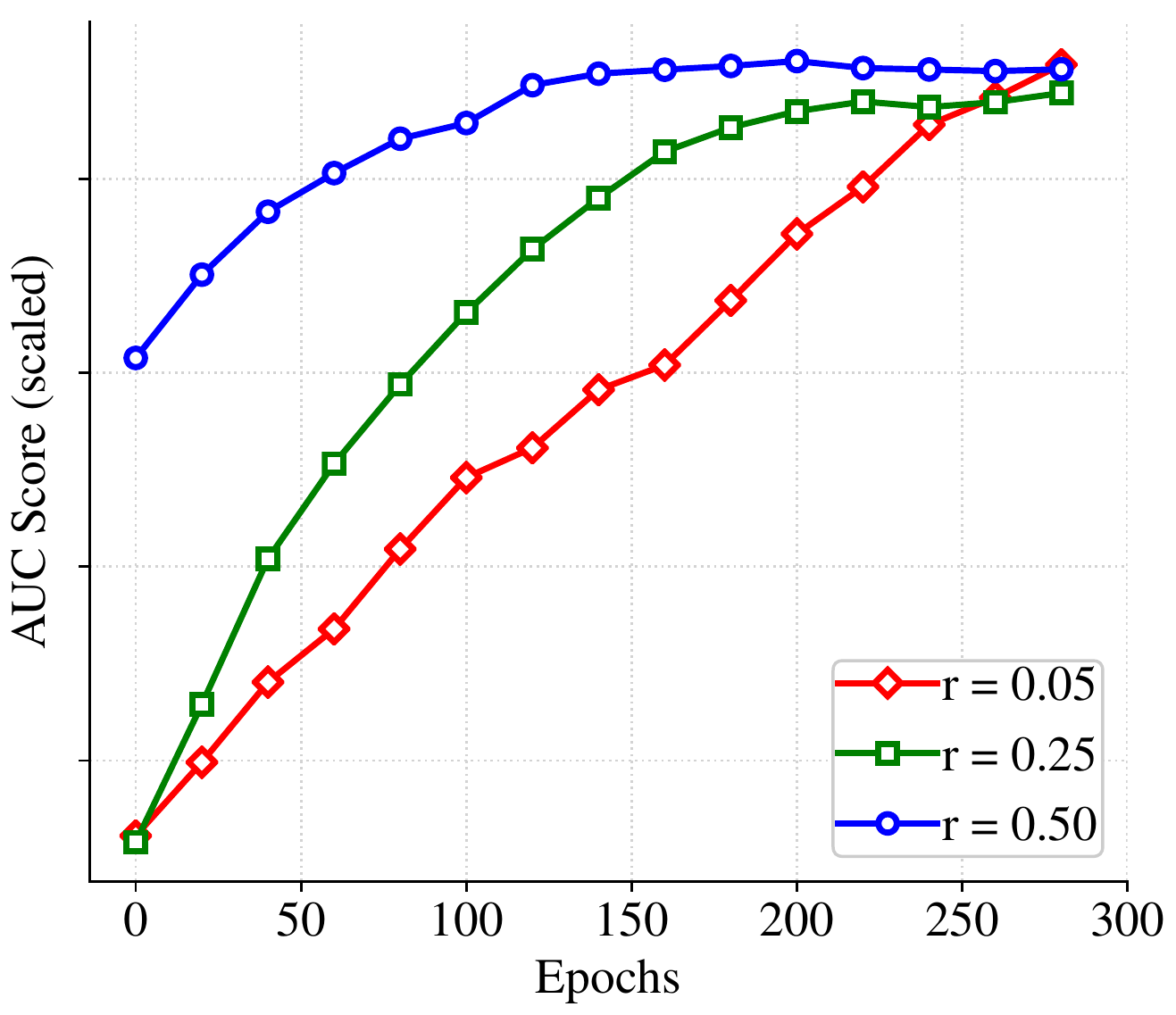}}
\hfill
\subfloat[$k_*=40$\label{fig:conv-40}]{\includegraphics[width=0.5\linewidth]{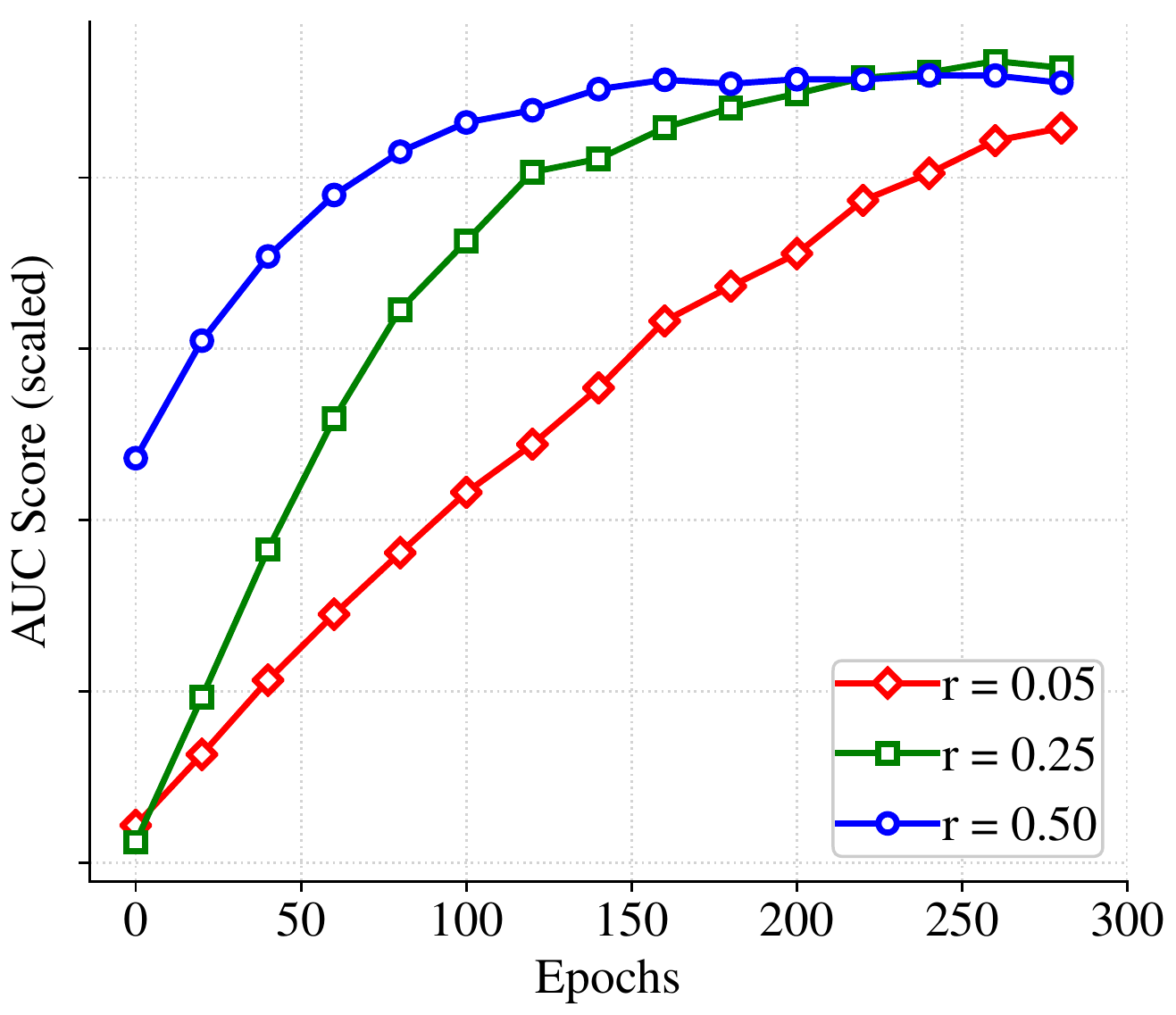}}\\
\subfloat[$k_*=60$\label{fig:conv-60}]{\includegraphics[width=0.5\linewidth]{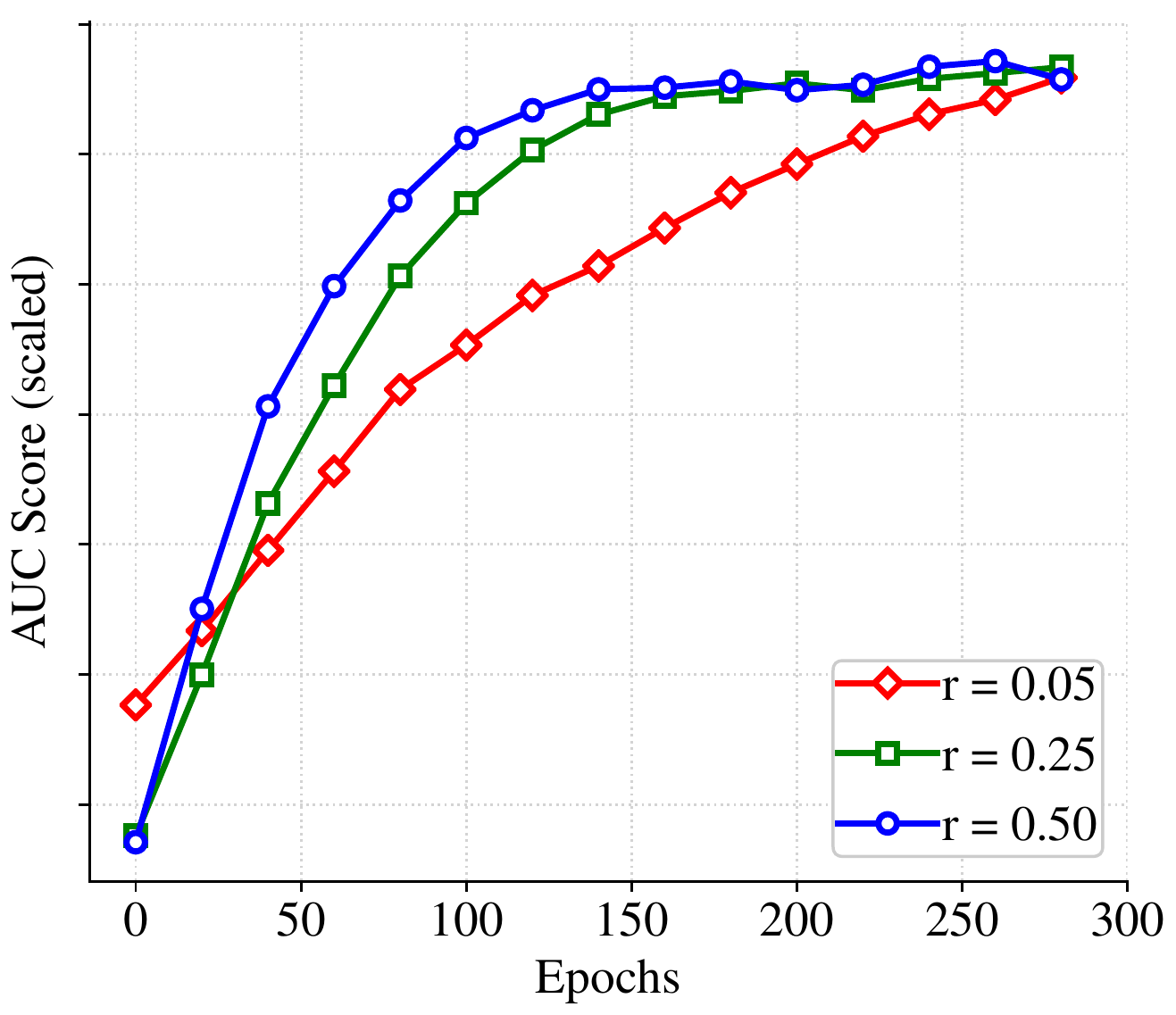}}
\hfill
\subfloat[$k_*=80$\label{fig:conv-80}]{\includegraphics[width=0.5\linewidth]{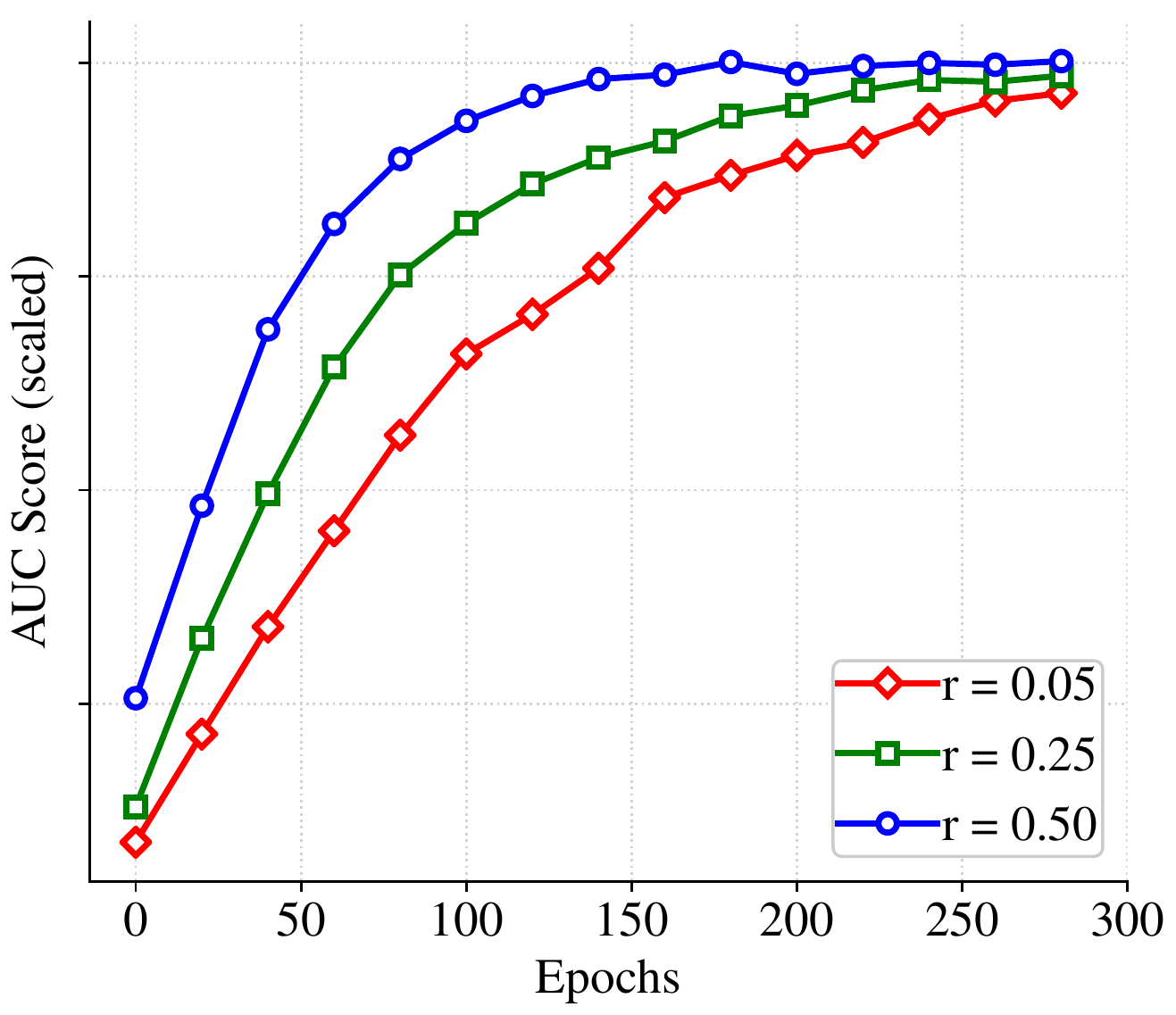}}
\caption{Convergence plot with different imbalance ratio $r$.}
\label{fig:conv}
\end{figure}

To further investigate how the imbalance ratio $r$ affects the performance, we apply all methods with different imbalance ratios on $k_*=20$ dataset.  As shown in Figure~\ref{fig:auc-score-pr}, for \shtauc, the more data is imbalanced, the worse the AUC and F1 scores are. This again proves our thereotical analysis in Section~\ref{subsection:Estimation of RCS and RSS Conditions}. Other AUC maximization algorithms also show the same phenomenon, but they are lack of similar analysis on imbalance ratio. Moreover, the performance of \shtauc,\textsc{SPAM}-$\ell_1$,\textsc{SPAM}-$\ell_1/\ell^2$, and \textsc{FSAUC} are at the same tier. The results of \textsc{SPAM}-$\ell^2$ and \textsc{SOLAM} are inferior to the hard thresholding-based and $\ell_1$ based methods. The reason is these two methods do not explore sparsity. More interestingly, compared with the methods (i.e. \textsc{StoIHT} and \textsc{HSG-HT}) for Empirical Risk Minimization, the AUC optimization-based methods outperform these two by a large margin when the dataset is more imbalanced. This testifies that minimizing the empirical risk loss may not lead to the best possible AUC values as stated in~\cite{cortes2004auc}.

\begin{figure}[ht]
    \centering
    \includegraphics[width=\linewidth]{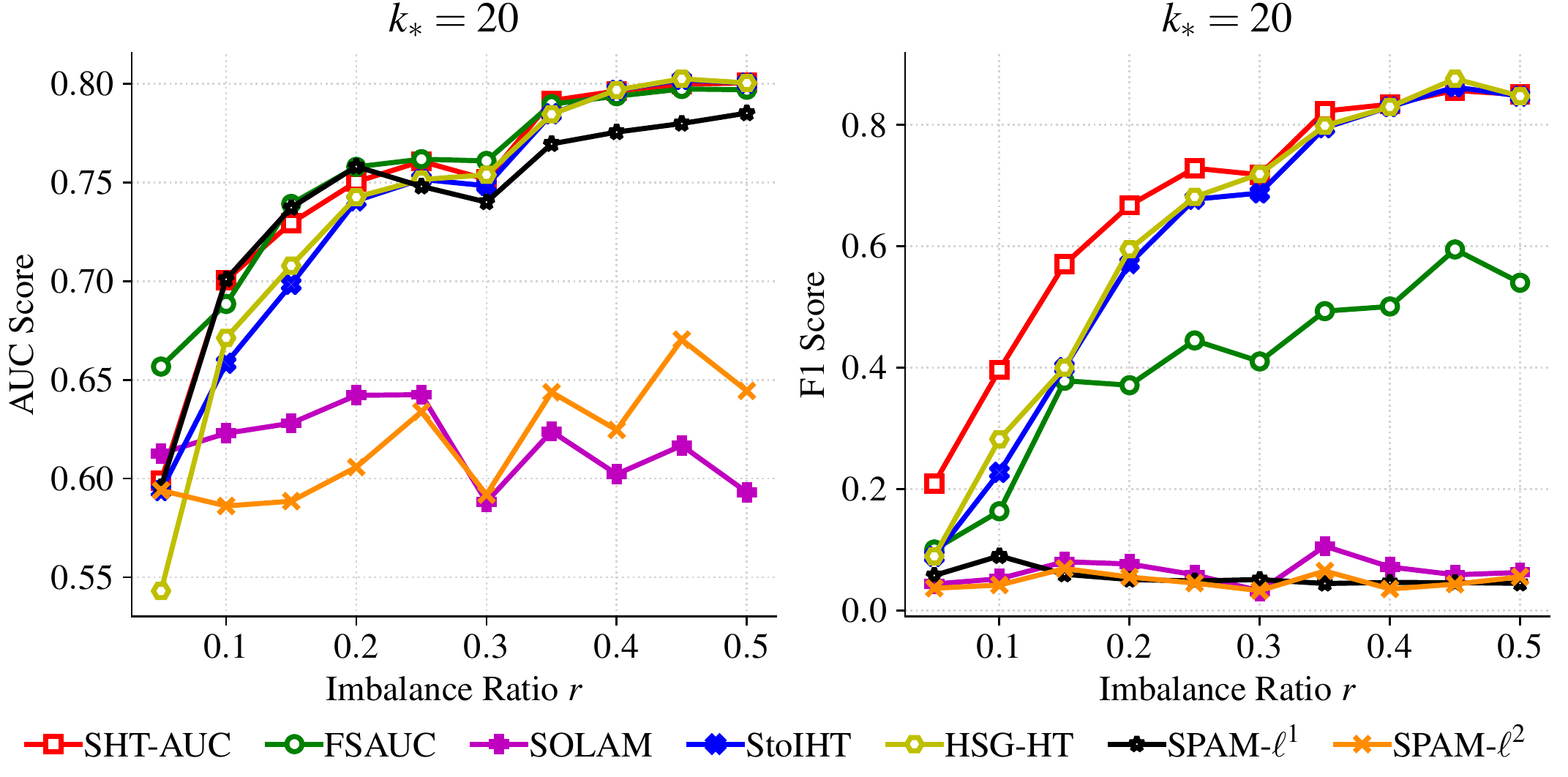}
    \caption{The left: AUC score as a function of the imbalance ratio $r$. The right: F1 score as a function of the imbalance ratio $r$.}
    \label{fig:auc-score-pr}\vspace*{-3mm}
\end{figure}


\subsection{Gene identification on two real-world datasets}

We test our method on two real-world high-dimensional datasets, the leukemia dataset~\cite{Golub1999un} and the colon cancer dataset~\cite{Alon1999dy}. The leukemia dataset consists of 72 samples where each positive sample (47 in total) is a patient has acute lymphoblastic leukemia and each negative sample (25 in total) is a patient has acute myeloid leukemia. Each training example has 7,129 genes. The colon cancer dataset has 62 training samples with 40 positive samples (patients who have tumor tissues) and 22 negative (patients who are normal). Each training sample consists 2,000 gene markers. Our goal is to classify these patients at the same time to select genes related with these two disease. As shown in Table~\ref{tab:myeloid-datasets} and~\ref{tab:colon-cancer}, we choose a subset of ground truth of cancer related genes from~\cite{agarwal2009ranking} and compare the gene selection ability of different methods.

\textbf{Parameter Tuning.} For $\shtauc$, \textsc{StoIHT}, and \textsc{HSG-HT}, the sparsity parameter $k$ is tuned from $\{1, 5, 10, \ldots,50, 60, \ldots,100, 200, \ldots,500\}$. For \textsc{SPAM}-$\ell^1$ and \textsc{SPAM}-$\ell^1/\ell^2$, we choose the $\ell_1$-regularization parameter $\lambda_{\ell_1}\in [0.07,0.00001]$ such that models are from sparsest models to dense models. \textsc{FSAUC} has a parameter $R$ to control $\ell_1$ ball, we choose $R\in [0.00001, 10000]$ such that models are from sparsest models to dense models too. \textsc{SPAM}-$\ell_2$ and \textsc{SOLAM} are two non-sparse methods. We randomly shuffle the dataset 20 times which form 20 trials. For each trial, we use 5-fold cross-validation to train all methods. The block size $b$ of \shtauc and \textsc{StoIHT} is tuned from 1 to 40 and sparsity $k$ is from 5 to 500.

\begin{table}[ht!]
\centering
\begin{tabular}{c|c|c}
\hline\hline
 & Colon Cancer & Leukemia \\\hline
\shtauc & \textbf{.8777 $\pm$ .1114} & \textbf{.9963$\pm$.0098} \\\hline
SPAM-$\displaystyle \ell ^{1}$ & .8409 $\pm$ .1646 & .9812$\pm$.0602  \\\hline
SPAM-$\displaystyle \ell ^{2}$ & .8304 $\pm$ .1478 & .9812$\pm$.0604 \\\hline
FSAUC & .7907 $\pm$ .2143 & .9708$\pm$.0730 \\\hline
SOLAM & .8089 $\pm$ .1752 & .9751$\pm$.0773 \\\hline
StoIHT & .8647 $\pm$ .1339 & .9947$\pm$.0138 \\\hline
HSG-HT & .8759 $\pm$ .1246 & .9898$\pm$.0218 \\\hline
\end{tabular}\caption{Average AUC score on real datasets}
\label{table:real-auc}
\end{table}

\begin{figure}[ht!]
\centering
\subfloat[Colon Cancer Dataset\label{fig:colon_auc}]{\includegraphics[width=.5\linewidth]{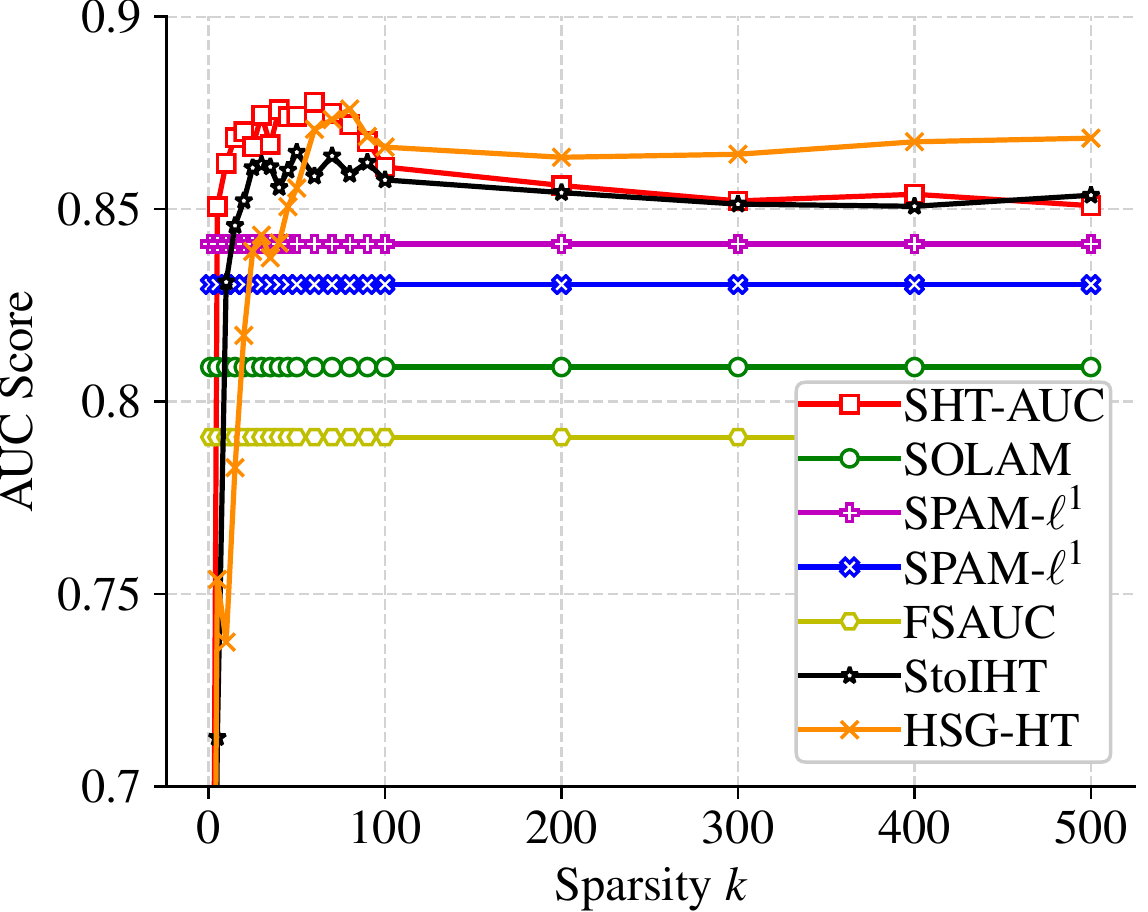}}
\subfloat[Leukemia Dataset\label{fig:leukemia_auc}]{\includegraphics[width=.5\linewidth]{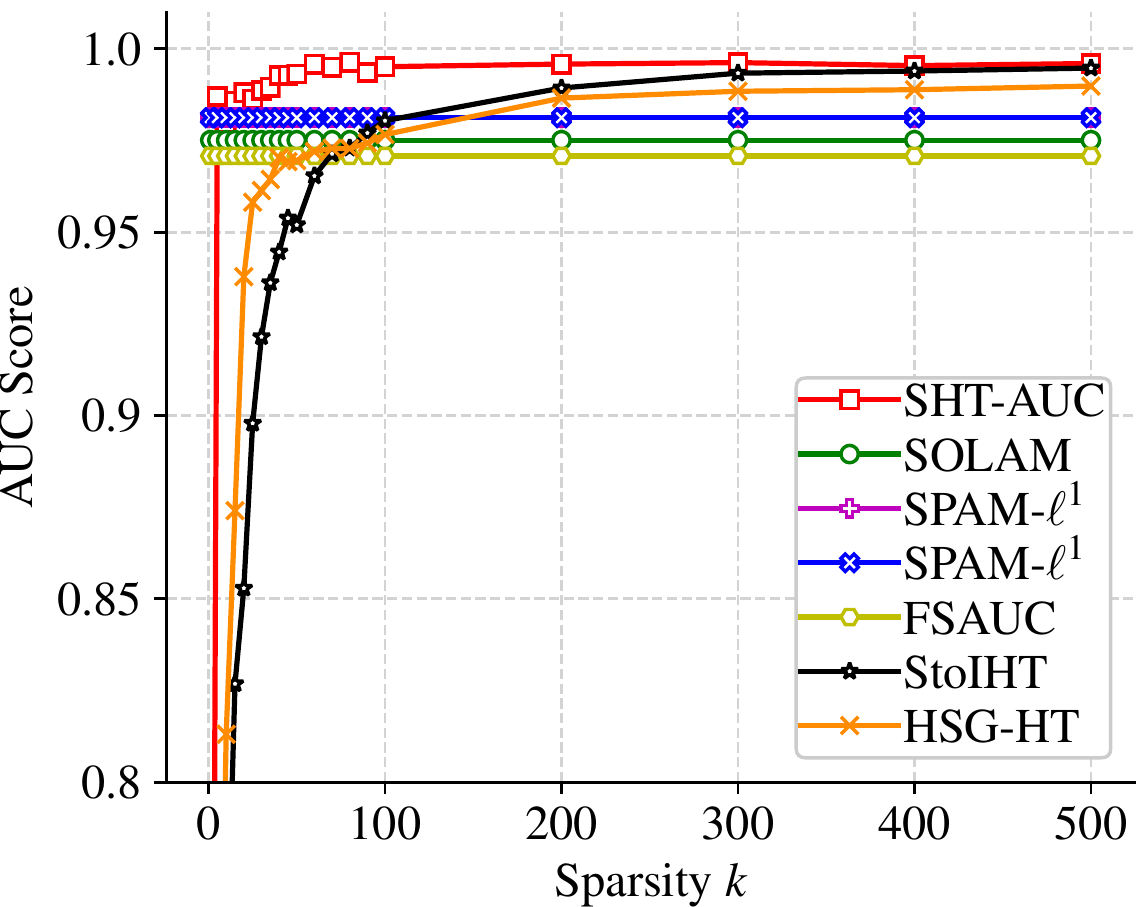}}
\caption{AUC score as a function of sparsity $k$.}
\label{fig:real:auc}
\end{figure}

\begin{table*}[ht!]
    \centering
    \begin{tabular}{p{0.03\textwidth}|p{0.4\textwidth}|p{0.03\textwidth}|p{0.5\textwidth}}\hline 
\rowcolor{lightgray}ID & Marker & ID & Marker\\\hline 
1 & Myeloperoxidase & 17 & Probable protein disulfide isomerase ER-60 precursor  \\\hline 
2 & CD13 & 18 & CD34 \\\hline 
3 & CD33 & 19 & CD24 \\\hline 
4 & HOXA9 Homeo box A9 & 20 &  60S ribosomal protein L23 \\\hline 
5 & MYBL2 & 21 & 5-aminolevulinic acid synthase \\\hline 
6 & CD19 & 22 &  HLA class II histocompatibility antigen \\\hline  
7 & CD10 (CALLA) & 23 &  Epstein-Barr virus small RNA-associated protein \\\hline 
8 & TCL1 (T cell leukemia) & 24 &  HNRPA1 Heterogeneous nuclear ribonucleoprotein A1 \\\hline
9 & C-myb &  25 & Azurocidin \\\hline 
10 & Deoxyhypusine synthase & 26 & Red cell anion exchanger (EPB3, AE1, Band 3)  \\\hline 
11 & KIAA0220  & 27 &  Topoisomerase II beta \\\hline 
12 & G-gamma globin & 28 & Probable G protein-coupled receptor LCR1 homolog \\\hline
13 &  Delta-globin  & 29 & Int-6 \\\hline 
14 & Brain-expressed HHCPA78 homolog & 30 & Alpha-tubulin \\\hline 
15 &  Myeloperoxidase & 31 & Terminal transferase \\\hline 
16 & NPM1 Nucleophosmin & 32 & Glycophorin B precursor \\\hline 
\end{tabular}
\caption{Markers related with acute myeloid leukemia and acute lymphoblastic leukemia.}
\label{tab:myeloid-datasets}
\end{table*}

\begin{table*}[ht!]
\centering
\begin{tabular}{p{0.03\textwidth}|p{0.47\textwidth}|p{0.03\textwidth}|p{0.4\textwidth}}\hline 
 \rowcolor{lightgray} ID & Marker                    & ID & Marker \\\hline 
 1 & Phospholipase A2           & 16 &  Splicing factor (CC1.4) \\\hline
 2 & Keratin 6 isoform          & 17 &  Nucleolar protein (B23) \\\hline
 3 & Protein-tyrosine phosphatase PTP-H1 & 18 & Lactate dehydrogenase-A (LDH-A) \\\hline
 4 & Transcription factor IIIA  & 19 & Guanine nucleotide-binding protein G(OLF) \\\hline 
 5 & Viral (v-raf) oncogene homolog 1 & 20 & LI-cadherin \\\hline 
 6 & Dual specificity mitogen-activated protein kinase kinase 1 & 21 & Lysozyme \\\hline 
 7 & Transmembrane carcinoembryonic antigen & 22 & Prolyl 4-hydroxylase (P4HB) \\\hline 
 8 & Oncoprotein 18 & 23 & Eukaryotic initiation factor 4AII \\\hline 
 9 & Phosphoenolpyruvate carboxykinase & 24 & Interferon-inducible protein 1-8D  \\\hline 
 10 & Extracellular signal-regulated kinase 1 & 25 & Dipeptidase \\\hline 
 11 &  26 kDa cell surface protein TAPA-1 & 26 & Heat shock 27 kDa protein \\\hline 
 12 & Id1 & 27 & Tyrosine-protein kinase receptor TIE-1 precursor \\\hline 
 13 & Interferon-inducible protein 9-27 & 28 & Mitochondrial matrix protein P1 precursor \\\hline 
 14 & Nonspecific crossreacting antigen & 29 & Eukaryotic initiation factor EIF-4A homolog \\ \hline
 15 & cAMP response element regulatory protein (CREB2) & & \\\hline 
\end{tabular}
\caption{Markers related with colon cancer as shown in~\cite{agarwal2009ranking}}
\label{tab:colon-cancer}
\end{table*}

\textbf{Generalization Performance.} We compare our method \shtauc on AUC score with seven baseline methods on colon cancer dataset. As shown in Table~\ref{table:real-auc}, our algorithm \shtauc has highest AUC score among all methods. It again shows the advantage of our algorithm in maximizing AUC with sparsity constraint. Interestingly, Empirical Risk Minimization algorithms \textsc{StoIHT} and \textsc{HSG-HT} also give comparable AUC scores. The explanation is that both datasets are not severely imbalanced. The performance of $\textsc{SPAM}$, $\textsc{FSAUC}$ and $\textsc{SOLAM}$ are second tier. The results show the advantage of $\ell_0$-based over $\ell_1$-based stochastic algorithms in AUC maximization.

Figure~\ref{fig:real:auc} shows the AUC score against the relaxed sparsity level $k$\footnote{\textsc{SOLAM}, \textsc{FSAUC} and \textsc{SPAM} are drawn as constant lines with best performance}. In Figure~\ref{fig:colon_auc}, \shtauc, \textsc{StoIHT} and \textsc{HSG-HT} reach their highest AUC scores when $k$ is moderately larger than $k_*$. This matches the condition in Theorem~\ref{thm:StoIHT}. In Figure~\ref{fig:leukemia_auc}, the AUC scores are saturated after certain $k$ with respect to different algorithms. The reason is Leukemia dataset has much higher dimension than Colon Cancer dataset. Hence $k$ is still in a reasonable range, and it still matches the condition in Theorem~\ref{thm:StoIHT}.

\begin{figure}[ht!]
\subfloat[Colon Cancer Dataset\label{fig:colon_feature}]{\includegraphics[width=0.5\linewidth]{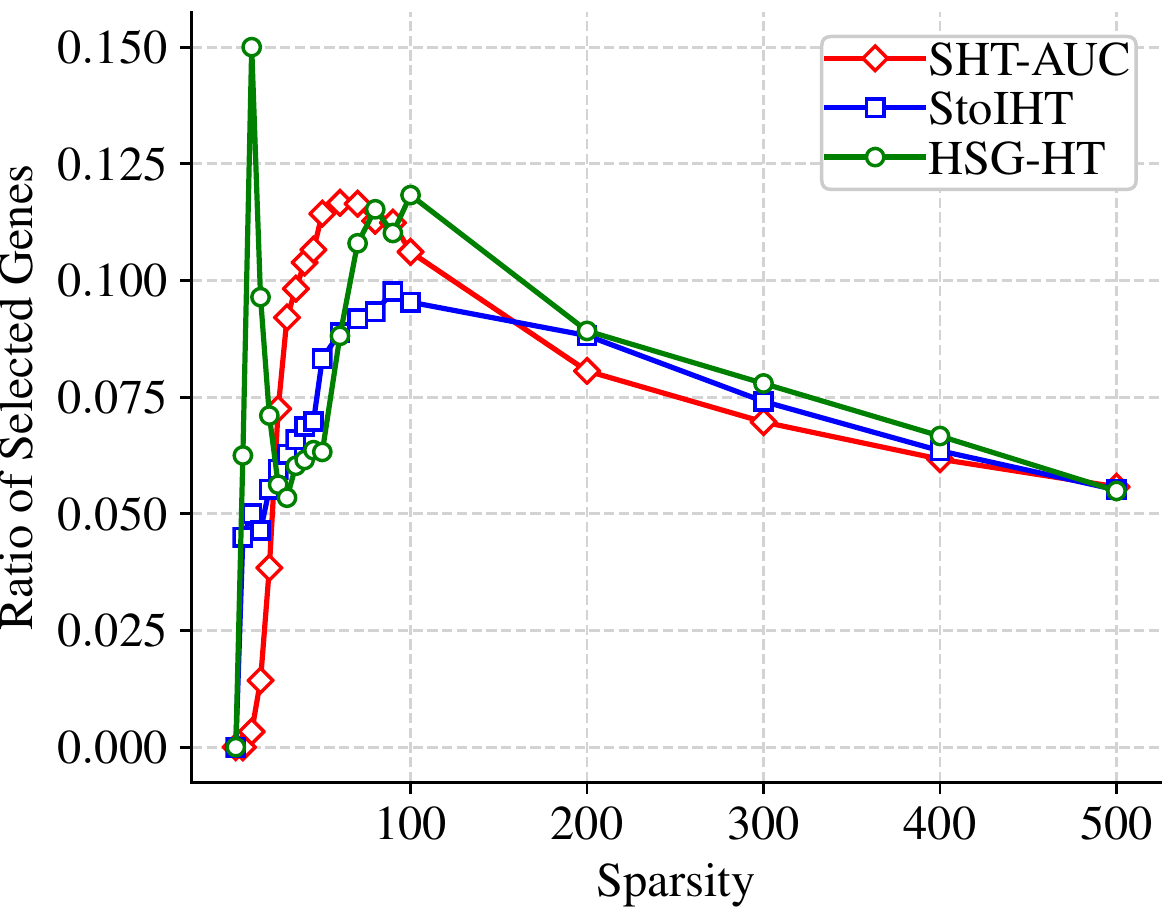}}
\hfill
\subfloat[Leukemia Dataset\label{fig:leukemia_feature}]{\includegraphics[width=0.5\linewidth]{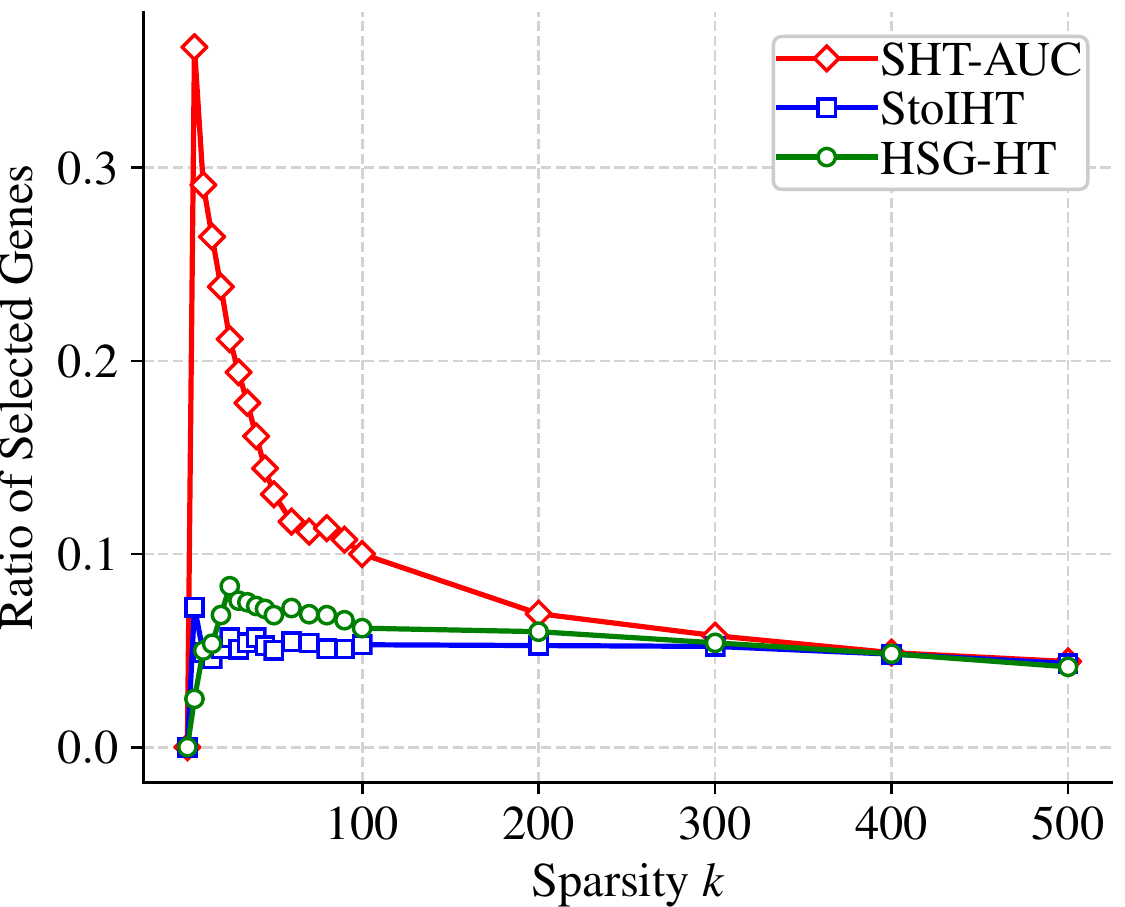}}
\caption{Percentage of related genes as a function as sparsity $k$ on the colon cancer and leukemia dataset. The ratio of selected genes is defined by the number of selected cancer-related genes divided by total number of genes found (corresponding to total number of non-zeros).}
\label{fig:real:feature}\vspace*{-3mm}
\end{figure}

\textbf{Feature Selection Ability.} To further investigate how is the gene selection ability for different methods on these datasets, we measure the gene selection ability as the ratio of related genes selected: the number of genes overlapped with the genes defined in Table~\ref{tab:myeloid-datasets} and~\ref{tab:colon-cancer} divided by total number of genes found. i.e. let $\bm w_t$ be the algorithm output and $\bm w_*$ be the groundtruth, the ratio is defined as
$$
\operatorname{Ratio}(\bm w_t, \bm w_*) = \frac{|\operatorname{supp}(\bm w_*)\cap \operatorname{supp}(\bm w_t)|}{|\operatorname{supp}(\bm w_*)|}.
$$
 We report our results in Figure~\ref{fig:real:feature}. In Figure~\ref{fig:colon_feature}, \textsc{HSG-HT} returns the best percentage when the sparsity is $k=5$, but when $k=29$, which is the number of related genes, \shtauc can achieve the highest percentage of related genes. Hence \shtauc and \textsc{HSG-HT} are comparable in general. This might due to the fact the colon cancer dataset is a relatively balanced balanced dataset, with $r = 0.355$. In Figure~\ref{fig:leukemia_feature}, \shtauc recovers the best percentage also when sparsity is the number of related genes, i.e. $k = 32$, and significantly outperforms \textsc{HSG-HT} and \textsc{StoIHT}. This phenomenon highlights the advantage of maximizing AUC rather than accuracy under imbalanced classification setting. 
 
 In summary, \shtauc enjoys better generalization performance on real-world high-dimensional datasets, while maintain a more robust feature selection ability against state-of-art algorithms.

\section{Conclusion}

In this paper, we proposed stochastic hard thresholding algorithm for AUC maximization with sparse $\ell_0$ contraints in imbalanced classification. In particular, we formulated the U-statistic objective function of AUC maximization as an ERM objective function. This new reformulation facilitated the design of stochastic hard thresholding algorithm for AUC maximization. The proposed algorithm, $\shtauc$, enjoys a cheap $\O(bd)$ per-iteration cost, making it amenable for high-dimensional data analysis. We proved that under RCS/RSS conditions, $\shtauc$ enjoys a linear convergence rate up to a tolerance error. We also showed, under Gaussian assumptions on the data, the RCS/RSS conditions can be satisfied and how the convergence rate and tolerance error are affected by the imbalance ratio. Our experiments validated our theoretical findings while $\shtauc$ is shown to have a very good property in feature selection against state-of-the-art algorithms. 

\section{Acknowledgement}

This work is supported by NSF IIS-1816227 and IIS-2008532. The work of Yunwen Lei is supported by the National Natural Science Foundation of China (Grant Nos. 61806091).

\bibliographystyle{IEEEtran}
\bibliography{reference}

\clearpage
\onecolumn

\section{Supplementary Material}
In this Supplementary Material, we provide the detailed proofs for Proposition \ref{prop:formulation}, Theorems \ref{thm:StoIHT} and \ref{thm:rip}.  

\subsection{Proof of Proposition~\ref{prop:formulation}}\label{sec:proof-prop1}
\begin{proof}
The objective function of AUC maximization given by ~\eqref{eq:single-sum-AUC} can be write in three terms,
\begin{align*}
    & F(\bm w)  = \frac{1}{n_+}\frac{1}{n_-}\sum_{i=1}^n\sum_{j=1}^n(1-\bm w^{\top}(\bm x_i-\bm x_j))^2\mathbb{I}_{[y_i=1]}\mathbb{I}_{[y_j=-1]}\\
    & = \frac{1}{n_+}\frac{1}{n_-} \sum_{i=1}^n\sum_{j=1}^n \left(1 + \bm w^\top \left(\overline{\bm x}_- - \overline{\bm x}_+\right) -  \bm w^\top \left(\bm x_i - \overline{\bm x}_+\right) + \bm w^\top \left(\bm x_j - \overline{\bm x}_-\right)\right)^2\mathbb{I}_{[y_i=1]}\mathbb{I}_{[y_j=-1]}\\
    & = \underbrace{\frac{1}{n_+}\frac{1}{n_-}\sum_{i=1}^n\sum_{j=1}^n\left(1 + \bm w^\top \left(\overline{\bm x}_- - \overline{\bm x}_+\right)\right)^2\mathbb{I}_{[y_i=1]}\mathbb{I}_{[y_j=-1]} }_{I}\\
    &\quad+ \underbrace{\frac{1}{n_+}\frac{1}{n_-}\sum_{i=1}^n\sum_{j=1}^n\left(\bm w^\top \left(\bm x_i - \overline{\bm x}_+\right) - \bm w^\top \left(\bm x_j - \overline{\bm x}_-\right)\right)^2\mathbb{I}_{[y_i=1]}\mathbb{I}_{[y_j=-1]}}_{II}\\
    &\quad+ \underbrace{\frac{1}{n_+}\frac{1}{n_-}\sum_{i=1}^n\sum_{j=1}^n2\left(1 + \bm w^\top \left(\overline{\bm x}_- - \overline{\bm x}_+\right)\right)\left(\bm w^\top \left(\bm x_i - \overline{\bm x}_+\right) - \bm w^\top \left(\bm x_j - \overline{\bm x}_-\right)\right)\mathbb{I}_{[y_i=1]}\mathbb{I}_{[y_j=-1]}}_{III}.
\end{align*}
It suffices to estimate the above terms one by one. To this end, the first term has $n_+n_-$ same terms, so $$I = \left(1 + \bm w^\top \left(\overline{\bm x}_- - \overline{\bm x}_+\right)\right)^2  = 1 + 2\bm w^\top \left(\overline{\bm x}_- - \overline{\bm x}_+\right) + \left(\bm w^\top \left(\overline{\bm x}_- - \overline{\bm x}_+\right)\right)^2.$$
For the second term, notice that the cross term \begin{align*}
&\frac{1}{n_+}\frac{1}{n_-}\sum_{i=1}^n\sum_{j=1}^n 2\bm w^\top \left(\bm x_i - \overline{\bm x}_+\right) \bm w^\top \left(\bm x_j - \overline{\bm x}_-\right)\mathbb{I}_{[y_i=1]}\mathbb{I}_{[y_j=-1]}\\
& = 2\left(\frac{1}{n_+}\sum_{i=1}^n \bm w^\top \left(\bm x_i - \overline{\bm x}_+\right)\mathbb{I}_{[y_i=1]}\right)\left(\frac{1}{n_-}\sum_{j=1}^n \bm w^\top \left(\bm x_j - \overline{\bm x}_-\right)\mathbb{I}_{[y_j=-1]}\right)\\
& = 2\left(\bm w^\top \left(\overline{\bm x}_+ - \overline{\bm x}_+\right)\right)\left(\bm w^\top \left( \overline{\bm x}_-- \overline{\bm x}_-\right)\right)\\
& = 0,
\end{align*}
we have 
\begin{align*}
    II& = \frac{1}{n_+}\frac{1}{n_-}\sum_{i=1}^n\sum_{j=1}^n \left(\bm w^\top \left(\bm x_i - \overline{\bm x}_+\right)\right)^2\mathbb{I}_{[y_i=1]}\mathbb{I}_{[y_j=-1]} +  \frac{1}{n_+}\frac{1}{n_-}\sum_{i=1}^n\sum_{j=1}^n\left(\bm w^\top \left(\bm x_j - \overline{\bm x}_-\right)\right)^2\mathbb{I}_{[y_i=1]}\mathbb{I}_{[y_j=-1]}\\& = \frac{1}{n_+}\sum_{i=1}^n \left(\bm w^\top \left(\bm x_i - \overline{\bm x}_+\right)\right)^2\mathbb{I}_{[y_i=1]} +  \frac{1}{n_-}\sum_{j=1}^n\left(\bm w^\top \left(\bm x_j - \overline{\bm x}_-\right)\right)^2\mathbb{I}_{[y_j=-1]}\\
    & = \frac{1}{n_+}\sum_{i=1}^n \left(\bm w^\top \left(\bm x_i - \overline{\bm x}_+\right)\right)^2\mathbb{I}_{[y_i=1]} +  \frac{1}{n_-}\sum_{i=1}^n\left(\bm w^\top \left(\bm x_i - \overline{\bm x}_-\right)\right)^2\mathbb{I}_{[y_i=-1]}.
\end{align*}
For the third term, by a similar argument of the cross term in the second term, we have $III = 0.$
Now the equation~\eqref{eq:single-sum-AUC} holds by proper scaling.
\end{proof}

\subsection{Proof of Theorem~\ref{thm:StoIHT}}\label{sec:proof-thm1}
Before we introduce the proof of the theorem we need several lemmas. 

\noindent The first lemma was originally proved in~\cite{nguyen2017linear}. It provides an estimate for our convergence analysis. Recall that we assume $\{f_{B_i}(\bm w) \}_{i=1}^m$ satisfies the RSS and $F(\bm w) = \sum_{i=1}^m f_{B_i}(\bm w)$ satisfies the RSC. 
\begin{lem}
\label{lem::StoIHT 1st corollary}
Let $i$ be an index selected with probability $1/n$ from the set $[n]$. For any fixed sparse vectors $\bm w$ and $\bm w'$, let $\Omega$ be a set such that $\supp(w) \cup \supp(w') \in \Omega$ and denote $s = |\Omega|$. We have
\begin{align*}
\label{inq::1st key observation}
&\E_i \norm{w' - w - \gamma \oper P_{\Omega} \left(\nabla f_{B_i}(w') - \nabla f_{B_i}(w) \right)}_2\leq \sqrt{1- (2 - \gamma \rho^+_s) \gamma \rho^-_s } \norm{w'-w}_2 \numberthis
\end{align*}

\end{lem}


The second lemma provides a refined bound on the deviation of the thresholded variable, which is originally proved in \cite{Shen2018}.

\begin{lem}
\label{lem:hard-thresholding-deviation}
Let $\bm w \in \mathbb{R}^d$ be an arbitrary vector and $\bm w^* \in \mathbb{R}^d$ be any $k^*$-sparse vector. For any $k \geq k^*$, we have the following bound:
\begin{equation*}
	\norm{\mathcal{H}_k(\bm w) - \bm w^*}_2 \leq \sqrt{1+\nu} \norm{\bm w - \bm w^*}_2,~\nu =  \frac{\mu + \sqrt{(4 + \mu ) \mu } }{2},~\mu =  \frac{\min\{k^*, d-k\}}{k - k^* + \min\{k^*, d-k\}}.
\end{equation*}
\end{lem}


\begin{proof} [Proof of Theorem~\ref{thm:StoIHT}]
By specifying $\Omega = \supp(\bm w_{t+1}) \cup \supp(\bm w_{t}) \cup \supp(\bm w_{*})$ and notice $|\Omega|\leq 2k + k_*$, it follows that $$\mathcal{H}_{2k+k_*}\left(\widehat{\bm w}_t\right) = \mathcal{H}_{2k+k_*}\left(\mathcal{P}_{\Omega}\left(\widehat{\bm w}_t\right)\right).$$ Thus, by the updating rule in Algorithm~\ref{alg:SG-HT} and  Lemma~\ref{lem:hard-thresholding-deviation} we have,

\begin{equation}
\begin{split}
\nonumber
& \norm{\bm w_{t+1} - \bm w_*}_2  = \norm{\mathcal{H}_{2k+k_*}\left(\mathcal{P}_{\Omega}\left(\widehat{\bm w}_t\right)\right)- \bm w_*}_2\\
&\leq \sqrt{1+\nu}\norm{\mathcal{P}_{\Omega}\left(\widehat{\bm w}_t\right) - \bm w_*}_2 \\
&= \sqrt{1+\nu}\norm{\bm w_t - \bm w_* - \gamma \mathcal{P}_{\Omega}\left(\nabla f_{B_{i_t}}(\bm w_t)\right)}_2 \\
&\leq \sqrt{1+\nu} (\norm{\gamma \mathcal{P}_{\Omega}\left(\nabla f_{B_{i_t}}(\bm w_*)\right)}_2 + \norm{\bm w_t - \bm w_* - \gamma \mathcal{P}_{\Omega}\left( \nabla f_{B_{i_t}}(\bm w_t) - \nabla f_{B_{i_t}}(\bm w_*) \right)}_2) \\
\end{split}
\end{equation}

where the second inequality holds because $\supp(\bm w_t - \bm w_*) \subseteq \Omega$, the second inequality holds because of triangle inequality.

Denote $I_t$ as the set containing all indices $i_1, i_2,..., i_t$ randomly selected at or before step $t$ of the algorithm: $I_t = \{i_1,...,i_t\}$. It is clear that $I_t$ determines the solutions $\bm w_1,...,\bm w_{t+1}$. We also denote the conditional expectation $\E_{i_t | I_{t-1}} \norm{\bm w_{t+1} - \bm w_*}_2 \triangleq \E_{i_t} (\norm{\bm w_{t+1} - \bm w_*}_2| I_{t-1})$. Now taking the conditional expectation on both sides of the above inequality  we obtain
\begin{equation}
\begin{split}
\nonumber
\E_{i_t | I_{t-1}} \norm{\bm w_{t+1} - \bm w_*}_2 \leq \sqrt{1+\nu}( \E_{i_t | I_{t-1}} \norm{\bm w_t - \bm w_* -\gamma \mathcal{P}_{\Omega}\left( \nabla f_{B_{i_t}}(\bm w_t) - \nabla f_{B_{i_t}}(\bm w_*) \right)}_2 \\ + \E_{i_t | I_{t-1}} \norm{\gamma \mathcal{P}_{\Omega}\left(\nabla f_{B_{i_t}}(\bm w_*)\right)}_2 ).
\end{split}
\end{equation}

Conditioning on $I_{t-1}$, $\bm w_t$ can be seen as a fixed vector. We apply the inequality (\ref{inq::1st key observation}) of Lemma \ref{lem::StoIHT 1st corollary}, we get 
\begin{equation}
\begin{split}
\nonumber
\E_{i_t | I_{t-1}} \norm{\bm w_{t+1} - \bm w_*}_2 & \leq \sqrt{(1+\nu)\left(1 -(2\gamma -\gamma^2\rho_{2k+k_*}^+)\rho_{2k+k_*}^-\right)} \norm{\bm w_t-\bm w_*}_2 \\ & + \sqrt{1+\nu} \gamma \E_{i_t}\norm{\mathcal{P}_{\Omega}\left( \nabla f_{B_{i_t}} (\bm w_*)\right)}_2 \\
&\leq \kappa \norm{\bm w_t -\bm w_*}_2 +  \sigma_{\bm w_*},
\end{split}
\end{equation}
where $\kappa$ and $\sigma_{w_{\star}}$ are defined in Theorem \ref{thm:StoIHT}. Taking the expectation on both sides with respect to $I_{t-1}$ yields
$$
\E_{I_t} \norm{\bm w_{t+1}-\bm w_*}_2 \leq \kappa \E_{I_{t-1}} \norm{\bm w_t-\bm w_*} + \sigma_{\bm w_*}.
$$

\noindent Applying this result recursively over $t$ iterations yields the desired result:
\begin{equation}
\begin{split}
\nonumber
\E_{I_t} \norm{\bm w_{t+1}-\bm w_*}_2 &\leq \kappa^{t+1} \norm{\bm w_0 - \bm w_*}_2 + \sum_{j=0}^t \kappa^j \sigma_{w_*} \\
&\leq \kappa^{t+1} \norm{\bm w_0 - \bm w_*}_2  + \frac{1}{1-\kappa} \sigma_{\bm w_*}.
\end{split}
\end{equation}
\end{proof}

\subsection{Proof of Theorems \ref{thm:rip}}\label{sec:proof-thm2}


In order to prove Theorems~\ref{thm:rip}, we need to following lemmas. Firstly, we introduce a lemma which is originally proved in \cite{Raskutti2009}.  This lemma captures the lower bound and upper bound of Gaussian random design matrix.
\begin{lem}\label{lem:design-matrix-bounds}
Consider a random design matrix $\bm X \in \mathbb{R}^{n\times d}$ formed by drawing each row $\bm x_i \in \mathbb{R}^d$ i.i.d. from an $N(0,\Sigma)$ distribution. Then for some positive constants $c_1$, $c_2$, $c_3$ and $c_4$,  we have for all $\bm v \in \mathbb{R}^d$ and $\bm v \in B_0(2k)$,
\begin{equation}\label{eq:positive-lower-bound}
	\frac{\|\bm X\bm v\|_2}{\sqrt{n}}\geq \left(\frac{\left\|\Sigma^{1/2}\bm v\right\|_2}{2\|\bm v\|_2} - 6\sqrt{2} \sqrt{\frac{\rho(\Sigma)k\log{d}}{n}} \right)\|\bm v\|_2
\end{equation}
 with probability $1 - \exp(-n/72)$.
\end{lem}

Secondly we include a elementary bound on the sum of ordered Gaussian variables, which is originally proved in \cite{Bellec2016}.

\begin{lem}\label{lemma:order-chi-square-bound}
	Let $g_1, ..., g_d$ be zero-mean Gaussian random variables with variance at most $\sigma^2$. Denote by $(g_{(1)}, ..., g_{(d)})$ be a non-increasing rearrangement of $(|g_1|, ..., |g_d|)$. Then
	\begin{equation}\label{eq:order-chi-square-bound}
		\mathbb{P}\left( \frac{1}{k\sigma^2}\sum_{j=1}^{k}g_{(j)}^2 > t\log\left(\frac{2d}{k}\right) \right) \leq \left(\frac{2d}{k}\right)^{1-\frac{3t}{8}}
	\end{equation}
	for all $t>0$ and $k \in \{1,...,d\}$.
\end{lem}


\begin{proof}[Proof of Theorem  \ref{thm:rip}]
Since $F$ it is a quadratic function of $\bm w$, we have 
\begin{align*} F(\bm w') - F(\bm w) - \left<\nabla F(\bm w),\bm w'-\bm w\right> = \frac{1}{2}(\bm w'-\bm w)^\top \nabla^2F(\bm w'-\bm w).
\end{align*} 
By the definition of $F$ in equation (\ref{eq:single-sum-AUC}), we have
\begin{align*}
\nabla^2F = & \frac{1}{n_+}\frac{1}{n_-}\sum_{i=1}^{n}\sum_{j=1}^{n}(\bm x_i-\bm x_j)(\bm x_i-\bm x_j)^\top\mathbb{I}_{[y_i=1]}\mathbb{I}_{[y_j=-1]}\\
= & \frac{1}{n^+}\frac{1}{n^-}\sum_{i=1}^{n^+}\sum_{j=1}^{n^-} (\bm x_i^+ - \bm x_j^-)(\bm x_i^+ - \bm x_j^-)^\top \\
= & \frac{1}{n^+}\sum_{i=1}^{n^+}\bm x_i^+\left(\bm x_i^+\right)^\top + \frac{1}{n^-}\sum_{j=1}^{n^-}\bm x_j^-\left(\bm x_j^-\right)^\top - \frac{1}{n^+}\frac{1}{n^-}\sum_{i=1}^{n^+}\sum_{j=1}^{n^-}\bm x_i^+\left(\bm x_j^-\right)^\top \\ & - \frac{1}{n^+}\frac{1}{n^-}\sum_{i=1}^{n^+}\sum_{j=1}^{n^-}\bm x_j^-\left(\bm x_i^+\right)^\top  \\
= & \frac{1}{n^+}\left(\bm X^+\right)^\top \bm X^+ + \frac{1}{n^-}\left(\bm X^-\right)^\top \bm X^- - \overline{\bm x}^+\left(\overline{\bm x}^-\right)^\top - \overline{\bm x}^-\left(\overline{\bm x}^+\right)^\top
\end{align*} 
By Lemma \ref{lem:design-matrix-bounds} we have for all $v \in \mathbb{R}^d$ and $v \in B_0(2k)$, $\bm X^+$ satisfies
\begin{align*}\label{eq:positive-matrix-lower-bound}
	\frac{\bm v^\top\left(\bm X^+\right)^\top \bm X^+\bm v}{n^+} \geq \left(\frac{\left\|\Sigma^{1/2}\bm v\right\|_2}{2\|\bm v\|_2} - 6\sqrt{2} \sqrt{\frac{\rho(\Sigma)k\log{d}}{n^+}} \right)^2\|\bm v\|_2^2,\numberthis
\end{align*}
with probability $1 - \exp(-n^+/72)$. And $\bm X^-$ satisfies
\begin{align*}\label{eq:negative-matrix-lower-bound}
\frac{\bm 
v^\top\left(\bm X^-\right)^\top \bm X^-\bm v}{n^-} \geq \left(\frac{\left\|\Sigma^{1/2}\bm v\right\|_2}{2\|\bm v\|_2} - 6\sqrt{2} \sqrt{\frac{\rho(\Sigma)k\log{d}}{n^-}} \right)^2\|\bm v\|_2^2,\numberthis
\end{align*}
with probability $1 - \exp(-n^-/72)$. 

Notice that
\begin{align*}
    \bm v^\top\left(\overline{\bm x}^+\left(\overline{\bm x}^-\right)^\top + \overline{\bm x}^-\left(\overline{\bm x}^+\right)^\top \right)\bm v & =  2\left(\bm v^\top \overline{\bm x}^+\right)\left(\bm v^\top \overline{\bm x}^-\right)\\ &\leq 2\left\|\mathcal{H}_k\left(\overline{\bm x}^+\right)\right\|_2\left\|\mathcal{H}_k\left(\overline{\bm x}^-\right)\right\|_2\|\bm v\|_2^2
\end{align*} 
for any $\|\bm v\|_0 \leq k$. Now by Lemma \ref{lemma:order-chi-square-bound} with $t = 16/3$ we have with probability $1 - k/2d$, 
\begin{align*}
    \left\|\mathcal{H}_k\left(\overline{\bm x}^+\right)\right\|_2^2 & \leq \frac{16 \rho(\Sigma)k\log (d)}{3n_+}\\
    \left\|\mathcal{H}_k\left(\overline{\bm x}^-\right)\right\|_2^2 & \leq \frac{16 \rho(\Sigma)k\log (d)}{3n_-}
\end{align*}
Therefore
\begin{align*}\label{eq:rank-one-matrix-upper-bound}
    \bm v^\top\left(\overline{\bm x}^+\left(\overline{\bm x}^-\right)^\top + \overline{\bm x}^-\left(\overline{\bm x}^+\right)^\top \right)\bm v 
    &\leq \frac{32}{3}\frac{\rho(\Sigma) k\log (d)}{\sqrt{n_+n_-}}\|\bm v\|_2^2\numberthis
\end{align*}

\noindent Now combine Equation (\ref{eq:positive-matrix-lower-bound}), (\ref{eq:negative-matrix-lower-bound}) and (\ref{eq:rank-one-matrix-upper-bound}) and rearrange, we have for any $w, w' \in B_0(k)$, 
\begin{align*}
&\frac{1}{2}(\bm w'-\bm w)^\top \nabla^2F(\bm w'-\bm w)\\
	& \geq  \Biggl(\left(\frac{\left\|\Sigma^{1/2}(\bm w'-\bm w)\right\|_2}{2\|\bm w'-\bm w\|_2} - 6\sqrt{2} \sqrt{\frac{\rho(\Sigma)k\log{d}}{n^+}}\right)^2 - \frac{32}{3}\frac{ \rho(\Sigma)k\log (d)}{\sqrt{n_+n_-}}\Biggr) \|\bm w'-\bm w\|_2^2 \\
	 & \geq \left(\left(\frac{1}{2}\lambda_{\min}(\Sigma^{1/2}) -  6\sqrt{2}\sqrt{\frac{\rho(\Sigma)k\log{d}}{rn}}\right)^2 - \frac{32}{3}\frac{ \rho(\Sigma)k\log (d)}{\sqrt{r(1-r)}n}\right) \|\bm w'-\bm w\|_2^2\\
	 & \geq \left(\left(\frac{1}{2}\lambda_{\min}(\Sigma^{1/2}) -  6\sqrt{2}\sqrt{\frac{\rho(\Sigma)k\log{d}}{rn}}\right)^2 - \frac{32}{3}\frac{ \rho(\Sigma)k\log (d)}{rn}\right) \|\bm w'-\bm w\|_2^2
\end{align*}
with probability  $(1 - \exp(-n^+/72))(1 - k/2d)$.


Now we turn to the estimate of $\rho_s^+$. Note that the restricted smoothness condition in equation (\ref{eq:rss}) is equivalent to $$f_{B_i}(\bm w') - f_{B_i}(\bm w) - \left<\nabla f_{B_i}(\bm w), \bm w' - \bm w\right> \leq \frac{\rho_k^+}{2}\|\bm w- \bm w'\|^2$$ for all vectors $\bm w$ and $\bm w'$ such that $|\supp(\bm w)\cup \supp(\bm w')|\leq k$. Also 
\begin{align*}f_{B_i}(\bm w') - f_{B_i}(\bm w) - \left<\nabla f_{B_i}(\bm w), \bm w' - \bm w\right> =\frac{1}{2}(\bm w'- \bm w)^\top \nabla^2 f_{B_i} (\bm w' - \bm w),
\end{align*} since $f_{B_i}$ is a quadratic function of $\bm w$. And  
\begin{align*}
    \nabla^2 f_{B_i} = & \frac{1}{b}\sum_{j \in B_i}\biggl(\frac{1}{r} (\bm x_j - \overline{\bm x}_+)(\bm x_j - \overline{\bm x}_+)^\top\mathbb{I}_{[y_j=1]} + \frac{1}{1-r} (\bm x_j - \overline{\bm x}_+)(\bm x_j - \overline{\bm x}_-)^\top\mathbb{I}_{[y_j=-1]}\\ &  +  (\overline{\bm x}_+ - \overline{\bm x}_+)( \overline{\bm x}_- - \overline{\bm x}_+)^\top\biggr)\\
    = & \frac{1}{b}\sum_{j \in B_i^+} \frac{1}{r} (\bm x_j^+ - \overline{\bm x}_+)(\bm x_j^+ - \overline{\bm x}_+)^\top + \frac{1}{b}\sum_{j \in B_i^-}\frac{1}{1-r} (\bm x_j^- - \overline{\bm x}_+)(\bm x_j^- - \overline{\bm x}_-)^\top \\  & + (\overline{\bm x}_- - \overline{\bm x}_+)( \overline{\bm x}_- - \overline{\bm x}_+)^\top
\end{align*}
Let $\widetilde{\bm x}_j \in \left\{\bm x_j^+ - \overline{\bm x}_+, \bm x_j^- - \overline{\bm x}_-, \overline{\bm x}_- - \overline{\bm x}_+\right\}$, since $\|\bm v\|_0 \leq k$, we have $$\bm v \widetilde{\bm x}_j \widetilde{\bm x}_j^\top \bm v = \left(\bm v^\top \widetilde{\bm x}_j\right)^2 \leq \left\|\mathcal{H}_k \left(\widetilde{\bm x}_j\right)\right\|_2^2\|\bm v\|_2^2 .$$
Also notice that by the properties of mean and variance, we have $\bm x_j^+ - \overline{\bm x}_+ \sim N\left(0, \frac{n_+-1}{n_+}\Sigma\right)$, $\bm x_j^- - \overline{\bm x}_- \sim N\left(0, \frac{n_--1}{n_-}\Sigma\right)$ and $\overline{\bm x}_- - \overline{\bm x}_+ \sim N\left(0,\left(\frac{1}{n_+} + \frac{1}{n_-}\right)\Sigma\right)$. Hence, by Lemma \ref{lemma:order-chi-square-bound}, pick $t = 16 \log\left(2\sqrt{b}d\right) / 3\log(2d)$ we have with probability $1-k/2bd$, 
$$\left\|\mathcal{H}_k\left(\bm x_j^+ - \overline{\bm x}_+\right)\right\|_2^2 \leq \frac{16}{3}\rho(\Sigma)\left(\frac{n_+-1}{n_+}\right)k\log(d)\left(\frac{1}{2}\log(b) + \log(d)\right),$$
$$\left\|\mathcal{H}_k\left(\bm x_j^- - \overline{\bm x}_-\right)\right\|_2^2 \leq \frac{16}{3}\rho(\Sigma)\left(\frac{n_--1}{n_-}\right)k\log(d)\left(\frac{1}{2}\log(b) + \log(d)\right),$$
$$\|\mathcal{H}_k(\overline{\bm x}_+ - \overline{\bm x}_-)\|_2^2 \leq \frac{16}{3}v\left(\frac{1}{n_+} + \frac{1}{n_-}\right)k\log(d)\left(\frac{1}{2}\log(b) + \log(d)\right).$$ 
Define the imbalance ratio in each batch $r_i = \frac{b_i^+}{b}$, therefore we have for any $\bm w, \bm w' \in B_0(k)$ with probability $1-k/2d$, 
\begin{align*}
&\max_i\left\{\frac{1}{2}(\bm w- \bm w')^\top \nabla^2f_{B_i}(\bm w- \bm w')\right\} \\
\leq & \max_i\biggl\{ \frac{16}{3}\left(\rho(\Sigma)\frac{r_i(n_+-1)}{rn_+}+\rho(\Sigma)\frac{(1-r_i)(n_--1)}{(1-r)n_-} + \rho(\Sigma)\frac{1}{n_+} + \rho(\Sigma)\frac{1}{n_-} \right)k\log(d)\\
& \times \left(\frac{1}{2}\log(b) + \log(d)\right) \left\|\bm w- \bm w'\right\|_2^2\biggr\}\\
\leq & \max_i\left\{ \frac{16}{3}\rho(\Sigma)\left(\frac{r_i}{r}  + \frac{1-r_i}{1-r} + \frac{1}{rn}+ \frac{1}{(1-r)n}\right)k\log(d)\left(\frac{1}{2}\log(b) + \log(d)\right) \left\|\bm w- \bm w'\right\|_2^2\right\}\\
\leq & \frac{16}{3}\rho(\Sigma)\left(\frac{1}{r} + \frac{1}{rn} + \frac{1}{(1-r)n} \right)k\log(d)\left(\frac{1}{2}\log(b) + \log(d)\right) \left\|\bm w- \bm w'\right\|_2^2\\
\leq & \frac{16\rho(\Sigma)k\log(d)\left(\frac{1}{2}\log(b) + \log(d)\right)}{r} \left\|\bm w- \bm w'\right\|_2^2
\end{align*}
where in the third and last inequality we use the fact that $1-r \geq r$.
\end{proof}


\subsection{Full formula of \texorpdfstring{$\frac{\sigma_{\bm w_*}}{1 - \kappa}$}{} in Section~\ref{subsection:Estimation of RCS and RSS Conditions}}\label{sec:sigma-kappa}
The derivation of $\sigma_{\bm w_*}$ is similar to the derivation of $\rho_s^+$. Hence it is omitted here. The full formula is given as follow.

\begin{align*}
    & \frac{\sigma_{\bm w_*}}{1-\kappa} =  \frac{4r \|\bm w_*\|_2 + \sqrt{\frac{r}{2n\rho(\Sigma)(2k+k_*)\log(d)}}}{1 - \sqrt{(1+\nu)\left(1 - \frac{3\lambda^2}{128k\log(d)}r + \left(\frac{9\sqrt{2}\lambda}{16\sqrt{k\log(d)n_-}} + \frac{1}{n}\right)\sqrt{r} - \frac{27}{4n}\right)}}
\end{align*}

\end{document}